%% file: paper.tex
\documentclass[11pt]{article}
\usepackage[numbers]{natbib}
\pdfoutput=1
\usepackage{amsmath,amssymb,amsthm}
\usepackage{algorithm,algorithmic}
\usepackage[colorlinks=true,linkcolor=blue]{hyperref}
\usepackage[usenames,dvipsnames]{xcolor}
\usepackage{fullpage}
\usepackage{MnSymbol}
\include{header}

\begin{document}

\setlength{\parskip}{2mm}
\setlength{\parindent}{0pt}

\title{Online Learning with Predictable Sequences}

\author{ Alexander Rakhlin\\ University of Pennsylvania \and Karthik Sridharan\\
University of Pennsylvania}

\maketitle

\begin{abstract}
	We present methods for online linear optimization that take advantage of benign (as opposed to worst-case) sequences. Specifically if the sequence encountered by the learner is described well by a known ``predictable process'', the algorithms presented enjoy tighter bounds as compared to the typical worst case bounds. Additionally, the methods achieve the usual worst-case regret bounds if the sequence is not benign. Our approach can be seen as a way of adding \emph{prior knowledge} about the sequence within the paradigm of online learning. The setting is shown to encompass partial and side information. Variance and path-length bounds \cite{hazan2010extracting, Chiangetal12} can be seen as particular examples of online learning with simple predictable sequences.
	
	We further extend our methods and results to include competing with a set of possible predictable processes (models), that is ``learning'' the predictable process itself concurrently with using it to obtain better regret guarantees. We show that such model selection is possible under various assumptions on the available feedback. Our results suggest a promising direction of further research with potential applications to stock market and time series prediction. 
\end{abstract}

\input{intro}
\input{fullinfo}

\input{ipm}

\input{mirror}
\input{localnorms}

\input{bandit}

\input{learnmt}

\input{fpl}
\input{other}

\input{small_losses}

\input{dbltrick}

\newpage
\appendix
\section{Appendix}
\input{appendix}

 \section*{Acknowledgements}
We gratefully acknowledge the support of NSF under grants CAREER DMS-0954737 and CCF-1116928, as well as Dean's Research Fund.

\bibliographystyle{plain}
\bibliography{paper}

\end{document}

%% file: header.tex
\usepackage{graphicx}
\usepackage{amsmath,amsfonts,amssymb,amsthm}
\usepackage{graphicx}
\usepackage{color}

\newtheorem{theorem}{Theorem}
\newtheorem{lemma}[theorem]{Lemma}

\newtheorem{corollary}[theorem]{Corollary}
\newtheorem{assumption}{Assumption}

\newtheorem{example}{Example}
\newtheorem{remark}{Remark}

\theoremstyle{definition}

\newcommand{\frameit}[1]{
    \begin{center}
    \framebox[.95\columnwidth][c]{
        \begin{minipage}{.9\columnwidth}
        #1
		\vspace{-0mm}
        \end{minipage}
    }
    \end{center}
}

\DeclareMathOperator*{\Ex}{\vphantom{p}\mathbb{E}}
\let\inf\undef
\DeclareMathOperator*{\inf}{\vphantom{p}inf}
\let\sup\undef
\DeclareMathOperator*{\sup}{\vphantom{p}sup}

\newcommand{\mbf}[1]{\mathbf{#1}}
\newcommand{\mc}[1]{\mathcal{#1}}
\newcommand{\mrm}[1]{\mathrm{#1}}

\newcommand{\Mbar}{\bar{M}}
\newcommand{\ev}[1]{\ensuremath{\Lambda_{#1}}}

\newcommand{\norm}[1]{\left\|#1\right\|}

\newcommand{\sign}{\mrm{sign}}
\newcommand{\argmin}[1]{\underset{#1}{\mrm{argmin}} \ }
\newcommand{\argmax}[1]{\underset{#1}{\mrm{argmax}} \ }
\newcommand{\reals}{\mathbb{R}}
\newcommand{\E}[1]{\mathbb{E}\left[ #1 \right]} % brackets
\newcommand{\En}{\mathbb{E}}  % no brackets or braces
\newcommand{\Es}[2]{\Ex_{#1}\left[ #2 \right]} % subscript + brackets

\newcommand{\inner}[1]{\left\langle #1 \right\rangle}
\newcommand{\ip}[2]{\left<#1,#2\right>}

\newcommand{\ind}[1]{{\bf 1}\left\{#1\right\}}
\newcommand{\tr}{\ensuremath{{\scriptscriptstyle\mathsf{T}}}}

\newcommand{\Eunderone}[1]{\underset{#1}{\En}}
\newcommand{\Eunder}[2]{\underset{\underset{#1}{#2}}{\En}}

\newcommand\cB{\mathcal{B}}

\newcommand\cR{\mathcal{R}}
\newcommand\X{\mathcal{X}}
\newcommand\cX{\mathcal{X}}

\newcommand\F{\mathcal{F}}

\newcommand\Reg{\mbf{Reg}}

\renewcommand\P{\mbf{P}}

\newcommand\loss{\mathrm{loss}}
\newcommand{\Relax}[3]{\mbf{Rel}_{#1}\left(#2 \middle| #3 \right)}
\newcommand{\Rel}[2]{\mbf{Rel}_{#1}\left(#2 \right)}

\def\deq{\triangleq}

%% file: intro.tex
% !TEX root =  paper.tex

\section{Introduction}

No-regret methods are studied in a variety of fields, including learning theory, game theory, and information theory \cite{PLG}. These methods guarantee a certain level of performance in a sequential prediction problem, irrespective of the sequence being presented. While such ``protection'' against the worst case is often attractive, the bounds are naturally pessimistic. It is, therefore, desirable to develop algorithms that yield tighter bounds for ``more regular'' sequences, while still providing protection against worst-case sequences. Some successful results of this type have appeared in \cite{cesa2007improved, hazan2010extracting, hazan2009better, Chiangetal12, BarHazRak07} within the framework of prediction with expert advice and online convex optimization. 

In \cite{RakSriTew11nips}, a general game-theoretic formulation was put forth, with ``regularity'' of the sequence modeled as a set of restrictions on the possible moves of the adversary. Through a non-constructive theoretical analysis, the authors of \cite{RakSriTew11nips} pointed to the \emph{existence} of quite general regret-minimization strategies for benign sequences, but did not provide a computationally feasible method. In this paper, we present algorithms that achieve some of the regret bounds of \cite{RakSriTew11nips} for sequences that can be roughly described as
$$\mbox{sequence = predictable process + adversarial noise}$$

This paper focuses on the setting of online linear optimization. The results achieved in the full-information case carry over to online \emph{convex} optimization as well. To remind the reader of the setting, the online learning process is modeled as a repeated game with convex sets $\F$ and $\X$ for the learner and Nature, respectively. At each round $t=1,\ldots,T$, the learners chooses $f_t\in\F$ and observes the move $x_t\in \X$ of Nature. The learner suffers a loss of $\inner{f_t,x_t}$ and the goal is to minimize regret, defined as 
$$\Reg_T = \sum_{t=1}^T \inner{f_t,x_t}-\inf_{f\in\F}\sum_{t=1}^T \inner{f,x_t}.$$

There are a number of ways to model ``more regular'' sequences. Let us start with the following definition. Fix a sequence of functions $M_t:\cX^{t-1}\mapsto \cX$, for each $t\in \{1,\ldots,T\} \deq [T]$. These functions define a predictable process
$$M_1,~ M_2(x_1),~\ldots,~ M_T(x_1,\ldots,x_{T-1}) \ .$$ 
If, in fact, $x_t = M_t(x_1,\ldots,x_{t-1})$ for all $t$, one may view the sequence $\{x_t\}$ as a (noiseless) time series, or as an oblivious strategy of Nature. If we knew that the sequence given by Nature follows exactly this evolution, we should suffer no regret.

Suppose that we have a hunch that the actual sequence will be ``roughly'' given by this predictable process: $x_t \approx M_t(x_1,\ldots,x_{t-1})$. In other words, we suspect that the sequence is described as predictable process plus adversarial noise. Can we use this fact to incur smaller regret if our suspicion is correct? Ideally, we would like to ``pay'' only for the unpredictable part of the sequence.

\paragraph{Information-Theoretic Justification} Let us spend a minute explaining why such regret bounds are information-theoretically possible. The key is the following observation, made in \cite{RakSriTew11nips}. The non-constructive upper bounds on the minimax value of the online game involve a symmetrization step, which we state for simplicity of notation for the linear loss case with $\F$ and $\cX$ being dual unit balls:
\begin{align*}
	\sup_{x_{1},x'_{1}}\En_{\epsilon_{1}}\ldots \sup_{x_{T},x'_{T}}\En_{\epsilon_{T}}\left\| \sum_{t=1}^{T} \epsilon_t\Big(x'_t - x_t\Big)  \right\|_*
	&\leq 2\sup_{x_{1}}\En_{\epsilon_{1}}\ldots\sup_{x_{T}}\En_{\epsilon_{T}} \left\| \sum_{t=1}^{T} \epsilon_t x_t \right\|_*
\end{align*}
If we instead only consider sequences such that at any time $t\in[T]$, $x_t$ and $x'_t$ have to be $\sigma_t$-close to the predictable process $M_t(x_1,\ldots,x_{t-1})$, we can add and subtract the ``center'' $M_t$ on the left-hand side of the above equation and obtain tighter bounds \emph{\textbf{for free, irrespective of the form of $M_t(x_1,\ldots,x_{t-1})$}}. To make this observation more precise, let 
\begin{align}
	\label{eq:def_constraint}
	C_t = C_t(x_1,\ldots,x_{t-1})= \big\{x: \|x-M_t(x_1,\ldots,x_{t-1})\|_* \leq \sigma_t\big\}
\end{align}
be the set of allowed deviations from the predictable ``trend''. We then have a bound 
$$\sup_{x_{1},x'_1\in C_1}\En_{\epsilon_{1}}\ldots\sup_{x_{T},x'_T \in C_T}\En_{\epsilon_{T}} \left\| \sum_{t=1}^{T} \epsilon_t \Big(x'_t-M_t(x_1,\ldots,x_{t-1}) + M_t(x_1,\ldots,x_{t-1})-x_t\Big) \right\|_* \leq c\sqrt{\sum_{t=1}^T \sigma_t^2}$$
on the value of the game against such ``constrained'' sequences, where the constant $c$ depends on the smoothness of the norm. This short description only serves as a motivation, and the more precise statements about the value of a game against constrained adversaries can be found in \cite{RakSriTew11nips}.

The development so far is a good example of how a purely theoretical observation can point to existence of better prediction methods. What is even more surprising, for most of the methods presented below, the individual $\sigma_t$'s need not be known ahead of time except for their total sum $\sum_{t=1}^T \sigma_t^2$. Moreover, the latter sum need not be known in advance either, thanks to the standard doubling trick, and one can obtain upper bounds of 
\begin{align}
	\label{eq:form_of_var_bound}
	\sum_{t=1}^T \inner{f_t,x_t}-\inf_{f\in\F}\sum_{t=1}^T \inner{f,x_t} \leq c\sqrt{\sum_{t=1}^T \|x_t-M_t(x_1,\ldots,x_{t-1})\|_*^2}
\end{align}
on regret, for some problem-dependent constant $c$.

Let us now discuss several types of statistics $M_t$ that could be of interest.
\begin{example}
	Regret bounds in terms of $$M_t(x_1,\ldots,x_{t-1}) = x_{t-1}$$ are known as \emph{path length bounds} \cite{RakSriTew11nips,Chiangetal12}. Such bounds can be tighter than the pessimistic $O(\sqrt{T})$ bounds when the previous move of Nature is a good proxy for the next move.
	
	Regret bounds in terms of 
	$$M_t(x_1,\ldots,x_{t-1}) = \frac{1}{t-1}\sum_{s=1}^{t-1} x_s$$ 
	are known as \emph{variance bounds} \cite{cesa2007improved,hazan2009better,hazan2010extracting,RakSriTew11nips}. 
	One may also consider fading memory statistics
	$$M_t(x_1,\ldots,x_{t-1}) = \sum_{s=1}^{t-1} \alpha_s x_s, ~~~~~~\sum_{s=1}^{t-1} \alpha_s = 1, ~~~~~ \alpha_s\geq 0$$
	or even plug in an \emph{auto-regressive model}.
	
	If ``phases'' are expected in the data (e.g., stocks tend to go up in January), one may consider 
	$$M_t(x_1,\ldots,x_{t-1}) = x_{t-k}$$ for some phase length $k$. Alternatively, one may consider averaging of the past occurrences $T_j(t)\subset \{1,\ldots,t\}$ of the current phase $j$ to get a better predictive power:
	$$M_t(x_1,\ldots,x_{t-1}) = \sum_{s\in T_t} \alpha_s x_s  \ .$$
\end{example}

\paragraph{Interpreting the Bounds} The use of a predictable process $(M_t)_{t \ge 1}$ can be seen as a way of incorporating \emph{prior knowledge} about the sequence $(x_t)_{t \ge 1}$. Importantly, the bounds still provide the usual worst-case protection if the process $M_t$ does not predict the sequence well. To see this, observe that the bounds of the paper scale with
$\sqrt{\sum_{t=1}^T \|x_t-M_t\|_*^2} \le 2\max_{x\in\X} \|x\|\sqrt{T}$
which is only a factor of $2$ larger than the typical bounds. However when $M_t$'s do indeed predict $x_t$'s well we have low regret, a property we get almost for free. Notice that in all our analysis the predictable process $(M_t)_{t \ge 1}$ can be any arbitrary function of the past.

\paragraph{A More General Setting}
The predictable process $M_t$ has been written so far as a function of $x_1,\ldots,x_{t-1}$, as we assumed the setting of full-information online linear optimization (that is, $x_{t}$ is revealed to the learner after playing $f_t$). Whenever our algorithm is deterministic, we may reconstruct the sequence $f_1,\ldots,f_{t}$ given the sequence $x_1,\ldots,x_{t-1}$, and thus no explicit dependence of $M_t$ of the learner's moves are required. More generally, however, nothing prevents us from defining the predictable process $M_t$ as a function 
\begin{align}
	\label{eq:M-fun-general}
	M_t (I_1,\ldots,I_{t-1}, f_1,\ldots,f_{t-1}, q_{1},\ldots,q_{t-1})
\end{align}
where $I_s$ is the \emph{information} conveyed to the learner at step $s\in[T]$ (defined on the appropriate information space ${\mathcal I}$) and $q_s$ is the randomized strategy of the learner at time $s$. For instance, in the well-studied bandit framework, the feedback $I_s$ is defined as the scalar value of the loss $\inner{f_s,x_s}$, yet the actual move $x_s$ may remain unknown to the learner. More general partial information structures have also been studied in the literature. 

When $M_t$ is written in the form \eqref{eq:M-fun-general}, it becomes clear that one can consider scenarios well beyond the partial information models. For instance, the information $I_s$ might contain better or complete information about the past, thus modeling a delayed-feedback setting (see Section~\ref{sec:delayed_feedback}). Another idea is to consider a setting where $I_s$ contains some state information pertinent to the online learning problem.

The paper is organized as follows. In Section~\ref{sec:fullinfo}, we provide a number of algorithms for full-information online linear optimization, taking advantage of a given predictable process $M_t$. These methods can be seen as being ``optimistic'' about the sequence, incorporating $M_{t+1}$ into the calculation of the next decision as if it were the true. We then turn to the partial information scenario in Section~\ref{sec:bandit} and show how to use the full-information bounds together with estimation of the missing information. Along the way, we prove a bound for nonstochastic multiarmed bandits in terms of the loss of the best arm -- a result that does not appear to be available in the literature.  In Section~\ref{sec:learning} we turn to the question of  \emph{learning} $M_t$ itself during the prediction process. We present several scenarios which differ in the amount of feedback given to the learner. Finally, we consider delayed feedback and some other scenarios that fall under the general umbrella.

\begin{remark}
We remark that most of the regret bounds we present in this paper are of the form $A\eta^{-1}+B\eta\sum_{t=1}^T\|x_t-M_t\|_*^2$. If variation around the trend is known in advance, one may choose $\eta$ optimally to obtain the form in \eqref{eq:form_of_var_bound}. Otherwise, we employ the standard doubling trick which we provide for completeness in Section~\ref{sec:dbltrick}. The doubling trick sets $\eta$ in a data-dependent way to achieve \eqref{eq:form_of_var_bound} with a slightly worse constant.
\end{remark}

\paragraph{Notation:}
We use the notation $y_{t':t}$ to represent the sequence $y_{t'}, \ldots,y_{t}$. We also use the notation $x[i]$ to represent the $i^{th}$ element of vector $x$. We use the notation $x[1:c]$ to represent the $c$-dimensional vector $(x[1],\ldots,x[c])$. $D_R(f,f')$ is used to represent the Bregman divergence between $f$ and $f'$ w.r.t. function $R$. We denote the set $\{1,\ldots,T\}$ by $[T]$.

%% file: fullinfo.tex
% !TEX root =  paper.tex
\section{Full Information Methods}
\label{sec:fullinfo}

In this section we assume that the value $M_t$ is known at the beginning of round $t$: it is either calculated by the learner or conveyed by an external source. The first algorithm we present is a modification of the Follow the Regularized Leader (FTRL) method with a self-concordant regularizer. The advantage of this method is its simplicity and the close relationship to the standard FTRL. Next, we exhibit a Mirror Descent type method which can be seen as a generalization of the recent algorithm of \cite{Chiangetal12}. Later in the paper (in Section~\ref{sec:fpl-methods}) we also present full-information methods based on the idea of a random playout, developed in \cite{RakShaSri12arxiv} for the problem of regret minimization in the worst case. To the best of our knowledge, these results are the first variation-style bounds for Follow the Perturbed Leader (FPL) algorithms.

For all the methods presented below, we assume (without loss of generality) that $M_1=0$. Since we assume that $M_t$ can be calculated from the information provided to the learner or the value of $M_t$ is conveyed from outside, we do not write the dependence of $M_t$ on the past explicitly.

%% file: ipm.tex
% !TEX root =  paper.tex
\subsection{Follow the Regularized Leader with Self-Concordant Barrier}
\label{sec:self-conc}

Let $\F\subset\reals^d$ be a convex compact set and let $\cR$ be a self-concordant function for this set. Without loss of generality, suppose that $\min_{f\in\F} \cR(f) = 0$. Given $f\in\mbox{int}(\F)$, define the local norm $\| \cdot \|_{f}$ with respect to $\cR$ by $\| g \|_{f} \deq \sqrt{g^\tr (\nabla^2 \cR(f) )g}$. The dual norm is then $\| x \|^*_{f} = \sqrt{x^\tr (\nabla^2 \cR(f) )^{-1} x}$. Given the $f_t$ defined in the algorithm below, we use the shorthand $\|\cdot\|_t = \| \cdot \|_{f_t}$. 

Consider the following algorithm.

\frameit{
\textbf{Optimistic Follow the Regularized Leader}\\
Input: $\cR$ self-concordant barrier, learning rate $\eta>0$. Initialize $f_1 = \arg\min_{f\in\F} \cR(f)$.\\
At $t=1,\ldots,T$, predict $f_t$, observe $x_t$, and update
	$$f_{t+1} = \arg\min_{f\in\F}~ \eta\inner{f, \sum_{s=1}^t x_s + M_{t+1}} + \cR(f)$$
}

We notice that for $M_{t+1}=0$, the method reduces to the Follow the Regularized Leader (FTRL) algorithm of \cite{abernethy2008competing,abernethy2012interior}. When $M_{t+1}\neq 0$, the algorithm can be seen as ``guessing'' the next move and incorporating it into the objective. If the guess turns out to be correct, the method should suffer no regret, according to the ``be the leader'' analysis.

The following regret bound holds for the proposed algorithm:

\begin{lemma}
	\label{lem:self_conc_full_info}
		Let $\F\subset \reals^d$ be a convex compact set endowed with a self-concordant barrier $\cR$ with $\min_{f\in\F} \cR(f)=0$. For any strategy of Nature, the Optimistic FTRL algorithm yields, for any $f^*\in\F$, 
	\begin{align}
		\sum_{t=1}^T \inner{f_t,x_t} - \sum_{t=1}^T \inner{f^*,x_t}\leq \eta^{-1}\cR(f^*) + 2\eta\sum_{t=1}^T (\|x_t-M_t\|_t^*)^2
	\end{align}	
	as long as $\eta\|x_t-M_t\|_t^*<1/4$ for all $t$.
\end{lemma}

By the argument of \cite{abernethy2008competing,abernethy2012interior}, at the expense of an additive constant in the regret, the comparator $f^*$ can be taken from a smaller set, at a distance $1/T$ from the boundary. For such an $f^*$, we have $\cR(f^*) \leq \vartheta\log T$ where $\vartheta$ is a self-concordance parameter of $\cR$.

%% file: mirror.tex
% !TEX root =  paper.tex

\subsection{Mirror-Descent algorithm}

The next algorithm is a modification of a Mirror Descent (MD) method for regret minimization. Let $\cR$ be a $1$-strongly convex function with respect to a norm $\|\cdot\|$, and let $D_\cR(\cdot,\cdot)$ denote the Bregman divergence with respect to $\cR$. Let $\nabla \cR^*$ be the inverse of the gradient mapping $\nabla \cR$. Let $\|\cdot\|_*$ be the norm dual to $\|\cdot\|$. We do not require $\F$ and $\cX$ to be unit dual balls.

% If $M_{t+1}=0$, one recovers the Mirror Descent algorithm. To see this, recursively define 
% $$g_{t+1} = \nabla\cR^*(\nabla\cR(g_t)-\eta x_t)$$ 
% for all $t>1$ and $g_{1}= \arg\min_{g} \cR(g)$. The projection of $g_{t+1}$ onto $\F$ with respect to $D_\cR$ is precisely the standard Mirror Descent algorithm. Now, since $\nabla\cR(g_1)=0$, unwinding the recursion we get $g_{t+1} = \nabla \cR^*\left( -\eta\sum_{s=1}^t x_s \right)$. Hence, $f_{t+1}$ coincides with the Mirror Descent solution when $M_{t+1}=0$. 
% 
% 
% Given the definition of $g_{t+1}$, the update for $f_{t+1}$ can then be written as
% $$f'_{t+1} = \nabla\cR^*( \nabla\cR(g_{t+1}) - \eta M_{t+1})$$
% followed by a projection with respect to the Bregman divergence $D_\cR$. It is not difficult to verify that this gives the following equivalent form of the algorithm:
Consider the following algorithm:
\frameit{
\textbf{Optimistic Mirror Descent Algorithm }\\
Input: $\cR$ $1$-strongly convex w.r.t. $\|\cdot\|$, learning rate $\eta>0$\\
Initialize $f_1 = g_1 = \arg\min_{g} \cR(g)$\\
At $t=1,\ldots,T$, predict $f_t$ and update
\begin{align*}
	g_{t+1} &= \argmin{g \in \F} \eta \inner{g, x_t} + D_\cR(g, g_t) \\
	f_{t+1} &= \argmin{f \in \F} \eta \inner{f,M_{t+1}} + D_\cR(f,g_{t+1})
\end{align*}
}

Such a two-projection algorithm for the case $M_t=x_{t-1}$ has been exhibited recently in \cite{Chiangetal12}. 

\begin{lemma}
	\label{lem:two-step-MD}
	Let $\F$ be a convex set in a Banach space $\cB$ and $\cX$ be a convex set in the dual space $\cB^*$. Let $\cR:\cB\mapsto\reals$ be a $1$-strongly convex function on $\F$ with respect to some norm $\|\cdot\|$. For any strategy of Nature, the Optimistic Mirror Descent Algorithm yields, for any $f^*\in\F$,
	$$\sum_{t=1}^T \inner{f_t, x_t} - \sum_{t=1}^T \inner{f^*, x_t} \leq  \eta^{-1}R_{\max}^2 + \frac{\eta}{2} \sum_{t=1}^T \|x_t-M_t\|_*^2 $$
	where $R_{\max}^2 = \max_{f\in\F} \cR(f) - \min_{f\in\F} \cR(f)$. % Choosing $\eta = \frac{\sqrt{2}R_{\max}}{\sqrt{\sum_{t=1}^T \|x_t-M_t\|_*^2}}$ 
	% 	yields 
	% 	\begin{align}
	% 		\frac{1}{T}\sum_{t=1}^T \inner{f_t,x_t} - \inf_{f\in\F} \frac{1}{T}\sum_{t=1}^T \inner{f,x_t} \leq \frac{\sqrt{2}R_{\max}}{T}\sqrt{\sum_{t=1}^T\|x_t-M_t(x_1,\ldots,x_{t-1})\|_*^2} \ .
	% 	\end{align}
\end{lemma}

As mentioned before, the sum $\sum_{t=1}^T \|x_t-M_t\|_*^2$ need not be known in advance in order to set $\eta$, as the usual doubling trick can be employed. Both the Optimistic MD and Optimistic FTRL work in the setting of online convex optimization, where $x_t$'s are now gradients at the points chosen by the learner. Last but not least, notice that if the sequence is not following the trend $M_t$ as we hoped it would, we still obtain the same bounds as for the Mirror Descent (respectively, FTRL) algorithm, up to a constant. 

%% file: localnorms.tex
% !TEX root =  paper.tex
\subsubsection{Local Norms for Exponential Weights}

For completeness, we also exhibit a bound in terms of local norms for the case of $\F\subset \reals^d$ being the probability simplex and $\cX$ being the $\ell_\infty$ ball. In the case of bandit feedback, such bounds serve as a stepping stone to building a strategy that explores according to the local geometry of the set \cite{AbeRak09}. Letting $\cR(f) = \sum_{i=1}^d f(i) \log f(i) -1$, the Mirror Descent algorithm corresponds to the well-known Exponential Weights algorithm. We now show that one can also achieve a regret bound in terms of local norms defined through the Hessian $\nabla^2\cR(f)$, which is simply $\mbox{diag}(f(1)^{-1},\ldots,f(d)^{-1})$. To this end, let $\|g\|_t = \sqrt{g^\tr \nabla^2\cR(f_t) g}$ and $\|x\|_t^* = \sqrt{x \nabla^2\cR(f_t)^{-1} x}$.

\begin{lemma}
	\label{lem:exp-weights-local}
	The Optimistic Mirror Descent on the probability simplex enjoys, for any $f^*\in\F$,
	$$\sum_{t=1}^T \inner{f_t-f^*, x_t} \leq 2\eta \sum_{t=1}^T (\norm{x_t - M_t}_t^*)^2 + \frac{ \log d}{\eta}$$
	as long as $\eta\|x_t-M_t\|_\infty\leq 1/4$ at each step.
	%Choosing $\eta = \sqrt{\frac{\log d}{2\sum_{t=1}^T (\norm{x_t - M_t}_t^*)^2}}$ balances the two terms.
\end{lemma}

%% file: bandit.tex
\section{Methods for Partial and Bandit Information}
\label{sec:bandit}

We now turn to the setting of partial information and provide a generic estimation procedure along the lines of \cite{hazan2009better}. Here, we suppose that the learner receives only partial feedback $I_t$ which is simply the loss $\inner{f_t,x_t}$ incurred at round $t$. Once again, we suppose to have access to some predictable process $M_t$. Note the generality of this framework: in some cases we might postulate that $M_t$ needs to be calculated by the learner from the available information (which does not include the actual moves $x_t$); in other cases, however, we may assume that some statistic $M_t$ (such as some partial information about the past moves) is conveyed to the learner as a side information from an external source. For the methods we present, we simply assume availability of the value $M_t$.

As in Section~\ref{sec:self-conc}, we assume to have access to a self-concordant function $\cR$ for $\F$, with the self-concordance parameter $\vartheta$. Following \cite{abernethy2008competing}, at time $t$ we define\footnote{We caution the reader that the roles of $f_t$ and $x_t$ in \cite{abernethy2008competing,hazan2009better} are exactly the opposite. We decided to follow the notation of \cite{RakSriTew10nips,RakShaSri12arxiv}, where in the supervised learning case it is natural to view the move $f_t$ as a function.} our randomized strategy $q_t$ to be a uniform distribution on the eigenvectors of $\nabla^2\cR(h_{t})$ where $h_{t}\in\F$ is given by a full-information procedure as described below. The full-information procedure is simply Follow the Regularized Leader on the \emph{estimated} moves $\tilde{x}_1,\ldots,\tilde{x}_{t-1}$ constructed from the information $I_1,\ldots,I_{t-1}, f_1,\ldots,f_{t-1}, q_{1},\ldots,q_{t-1}$, with $I_s = \inner{f_s,x_s}$. The resulting algorithm, dubbed SCRiBLe in \cite{abernethy2012interior}, is presented below for completeness:
\newcommand{\spce}{\hphantom{.}\hspace{3mm}}
\frameit{
			{\bf SCRiBLe}  \cite{abernethy2012interior,abernethy2008competing}\\
			Input: $\eta > 0$, $\vartheta$-self-concordant $\cR$. Define $h_1 = \arg\min_{f \in \F} \cR(f) $. \\
			At time $t=1$ to $T$\\
		    \spce Let $ \{\ev{1}, \ldots, \ev{n}\} $ and $\{\lambda_1,\ldots, \lambda_n\}$ be the eigenvectors and eigenvalues of $\nabla^2 \cR(h_t)$.\\
		    \spce Choose $i_t$ uniformly at random from $\{1,\ldots, n\}$ and $\varepsilon_t = \pm 1$ with probability $1/2$.\\
			\spce Predict $f_t = h_t + \varepsilon_t \lambda_{i_t}^{-1/2} \ev{i_t}$ and observe loss $\inner{f_t, x_t}$. \\
		  	\spce Define $\tilde{x}_t := n \left(\inner{f_t, x_t} \right)\varepsilon_t \lambda_{i_t}^{1/2} \cdot \ev{i_t}$. \\
		  	\spce Update $$ h_{t+1} = \arg\min_{h \in \F} \left[\eta\inner{h, \sum_{s=1}^t \tilde{x}_s} + \cR(h)\right].$$
}

Hazan and Kale \cite{hazan2009better} observed that the above algorithm can be modified by adding and subtracting an estimated mean of the adversarial moves at appropriate steps of the method. We use this idea with a general process $M_t$:
%$\Mbar_t (\inner{f_1,x_1},\ldots,\inner{f_{t-1},x_{t-1}}, f_1,\ldots,f_{t-1}, q_{1},\ldots,q_{t-1})$. 

\frameit{ 
			{\bf SCRiBLe for a Predictable Process} \\
			Input: $\eta > 0$, $\vartheta$-self-concordant $\cR$. Define $h_1 = \arg\min_{f \in \F} \cR(f) $. \\
			At time $t=1$ to $T$\\
			\spce Let $ \{\ev{1}, \ldots, \ev{n}\} $ and $\{\lambda_1,\ldots, \lambda_n\}$ be the eigenvectors and eigenvalues of $\nabla^2 \cR(h_t)$.\\
		  	\spce Choose $i_t$ uniformly at random from $\{1,\ldots, n\}$ and $\varepsilon_t = \pm 1$ with probability $1/2$.\\
			\spce Predict $f_t = h_t + \varepsilon_t \lambda_{i_t}^{-1/2} \ev{i_t}$ and observe loss $\inner{f_t, x_t}$. \\
		  \spce Define $\tilde{x}_t := n \left(\inner{f_t, x_t-M_t} \right)\varepsilon_t \lambda_{i_t}^{1/2} \cdot \ev{i_t} + M_t$. \\
		  \spce Update $$ h_{t+1} = \arg\min_{h \in \F} \left[\eta\inner{h, \sum_{s=1}^t \tilde{x}_s + M_{t+1}} + \cR(h)\right].$$
}

The analysis of the method is based on the bounds for full information predictable processes $M_t$ developed earlier, thus simplifying and generalizing the analysis of \cite{hazan2009better}.

\begin{lemma}
	\label{lem:bandit}
	Suppose that $\F$ is contained in the $\ell_2$ ball of radius $1$. The expected regret of the above algorithm (SCRiBLe for a Predictable Process) is
	\begin{align}
		\En\left[ \sum_{t=1}^T \inner{f_t,x_t} - \sum_{t=1}^T \inner{f^*,x_t} \right] &\le \eta^{-1}\cR(f^*) + 2\eta n^2 \E{\sum_{t=1}^T  (\inner{f_t, x_t-M_t})^2}  \label{eq:bandit}\\
		&\leq \eta^{-1}\cR(f^*) + 2\eta n^2\sum_{t=1}^T  \En\left[\|x_t-M_t\|^2\right] \notag
	\end{align}
	Hence, for any full-information statistic $M'_t=M'_t(x_1,\ldots,x_{t-1})$, 
	\begin{align}
		\label{eq:three-term-partial-info}
		\En\left[ \sum_{t=1}^T \inner{f_t,x_t} - \sum_{t=1}^T \inner{f^*,x_t} \right]
		&\leq \eta^{-1}\cR(f^*) + 4\eta n^2\sum_{t=1}^T  \En\left[\|x_t-M'_t\|^2\right] + 4\eta n^2\sum_{t=1}^T  \En\left[\|M_t-M'_t\|^2\right]
	\end{align}
	
\end{lemma}

Effectively, Hazan and Kale show in \cite{hazan2009better} that for the full-information statistic $M'_t(x_1,\ldots,x_{t-1}) = \frac{1}{t-1}\sum_{s=1}^{t-1} x_s$, there is a way to construct $M_t = M_t(I_1,\ldots,I_{t-1}, f_1,\ldots,f_{t-1}, q_{1},\ldots,q_{t-1})$  in such a way that the third term in \eqref{eq:three-term-partial-info} is of the order of the second term. This is done by putting aside roughly $O(\log T)$ rounds in order to estimate $M'_t$, via a  process called \emph{reservoir sampling}. However, for more general functions $M'_t$, the third term might have nothing to do with the second term, and the investigation of which $M'_t$ can be well estimated by $M_t$ is an interesting topic of further research.

%% file: learnmt.tex
\section{Learning The Predictable Processes}
\label{sec:learning}

So far we have seen that the learner with an access to an arbitrary predictable process $(M_t)_{t \ge 1}$ has a strategy that suffers regret of 
$
O\left(\sqrt{\sum_{t=1}^T \|x_t-M_t\|_*^2}\right)~.
$
Now if the predictable process is a good predictor of the sequence, then the regret will be low. This raises the question of model selection: how can the learner \emph{choose} a good predictable process $(M_t)_{t \ge 1}$? Is it possible to learn it \emph{online} as we go, and if so, what does it mean to learn? 

To formalize the concept of learning the predictable process, let us consider the case where we have a set $\Pi$ indexing a set of predictable processes (strategies) we are interested in. That is, each $\pi \in \Pi$ corresponds to predictable process given by $(M^\pi_t)_{t \ge 1}$. Now if we had an oracle which in the start of the game told us which $\pi^* \in \Pi$ predicts the sequence optimally (in hindsight) then we could use the predictable process given by $(M^{\pi^*}_t)_{t \ge 1}$ and enjoy a regret bound of 
$$ O\left(\sqrt{\inf_{\pi \in \Pi} \sum_{t=1}^T \|x_t-M^\pi_t\|_*^2}\right) \ .$$ 
However we cannot expect to know which $\pi \in \Pi$ is the optimal one from the outset. In this scenario one would like to learn a predictable process that in turn can be used with algorithms proposed thus far to get a regret bound comparable with regret bound one could have obtained knowing the optimal $\pi^* \in \Pi$.

\subsection{Learning $M_t$'s : Full Information}
\label{subsec:fullinfo}

To motivate this setting better let us consider an example. Say there are $n$ stock options we can choose to invest in. On each day $t$, associated with each stock option one has a loss/payoff that occurs upon investing in a single share of that stock. Our goal in the long run is to have a low regret with respect to the single best stock in hindsight. Up to this point, the problem just corresponds to the simple experts setting where each of the $n$ stocks is one expert and on each day we split our investment according to a probability distribution over the $n$ options. However now additionally we allow the learner/investor access to \emph{prediction models} from the set $\Pi$. These could be human strategists making forecasts, or  outcomes of some hedge-fund model. At each time step the learner can query prediction made by each $\pi \in \Pi$ as to what the loss on the $n$ stocks would be on that day. Now we would like to have a regret comparable to the regret we can achieve knowing the best model $\pi^* \in \Pi$ that in hind-sight predicted the losses of each stock optimally. We shall now see how to achieve this.

\frameit{
\textbf{Optimistic Mirror Descent Algorithm with Learning the Predictable Process}\\
Input: $\cR$ $1$-strongly convex w.r.t. $\|\cdot\|$, learning rate $\eta>0$\\
Initialize $f_1 = g_1 = \arg\min_{g} \cR(g)$ and initialize $q_1 \in \Delta(\Pi)$ as, $\forall \pi \in \Pi, q_1(\pi) = \frac{1}{\left|\Pi\right|}$\\
Set $M_1 = \sum_{\pi \in \Pi} q_1(\pi) M_1^\pi$\\
At $t=1,\ldots,T$, predict $f_t$, observe $x_t$ and update
$$\forall \pi \in \Pi,~ q_{t+1}(\pi) \propto q_{t}(\pi)\, e^{- \norm{M^\pi_t - x_t}_*^2} ~~\textrm{ and }~~ M_{t+1} = \sum_{\pi \in \Pi} q_{t+1}(\pi) M_{t+1}^\pi$$
and
\begin{align*}
	g_{t+1} &= \argmin{g \in \F} \eta \inner{g, x_t} + D_\cR(g, g_t) \\
	f_{t+1} &= \argmin{f \in \F} \eta \inner{f,M_{t+1}} + D_\cR(f,g_{t+1})
\end{align*}
}

The proof of the following lemma relies on a particular regret bound of  \cite[Corollary 2.3]{PLG} for the exponential weights algorithm that is in terms of the loss of the best arm. Such a bound is an improvement over the pessimistic regret bound when the loss of the optimal arm is small.
\begin{lemma}
	\label{lem:two-step-MD-learn}
	Let $\F$ be a convex subset of a unit ball in a Banach space $\cB$ and $\cX$ be a convex subset of the dual unit ball. Let $\cR:\cB\mapsto\reals$ be a $1$-strongly convex function on $\F$ with respect to some norm $\|\cdot\|$. For any strategy of Nature, the Optimistic Mirror Descent Algorithm yields, for any $f^*\in\F$,
	$$\sum_{t=1}^T \inner{f_t, x_t} - \sum_{t=1}^T \inner{f^*, x_t} \leq  \eta^{-1}R_{\max}^2 + 3.2\, \eta\left( \inf_{\pi \in \Pi} \sum_{t=1}^T \|x_t-M^\pi_t\|_*^2 +  \log \left|\Pi \right|\right) $$
	where $R_{\max}^2 = \max_{f\in\F} \cR(f) - \min_{f\in\F} \cR(f)$. 
\end{lemma}

Once again, let us discuss what makes this setting different from the usual setting of experts. The forecast given by prediction models is in the form of a vector, one for each stock. If we treat each prediction model as an expert with the loss $\|x_t-M^\pi_t\|_*^2$, the experts algorithm would guarantee that we achieve the best cumulative loss of this kind. However, this is not the object of interest to us, as we are after the best allocation of our money among the stocks, as measured by $\inf_{f\in\F} \sum_{t=1}^T \inner{f, x_t}$.

The algorithm can be seen as separating two steps: learning the model (that is, predictable process) and then minimizing regret given the learned process. This is implemented by a general idea of running another (secondary) regret minimizing strategy where loss per round is simply $\norm{M_t - x_t}_*^2$ and regret is considered with respect to the best $\pi \in \Pi$. That is, regret of the secondary regret minimizing game is given by
$$
\sum_{t=1}^T \norm{x_t - M_t}_*^2  - \inf_{\pi \in \Pi} \sum_{t=1}^T \norm{x_t - M^\pi_t}_*^2 
$$
In general, the experts algorithm for minimizing secondary regret can be replaced by any other online learning algorithm.

\subsection{Learning $M_t$'s : Partial Information}
\label{subsubsec:partial_info}

In the previous section we considered the full information setting where on each round we have access to $x_t$ and for each $\pi$ we get to see (or compute) $M_t^\pi$. However one might be in a scenario with only partial access to $x_t$ or $M_t^\pi$, or both. In fact, there are quite a number of interesting partial-information scenarios, and we consider some of them in this section.

\subsubsection{Partial Information about Loss (Bandit Setting)}
\label{subsubsubsec:partial_info_loss}

In this setting at each time step $t$, we only observe the loss $\ip{f_t}{x_t}$ and not all of $x_t$. However, for each $\pi \in \Pi$ we do get access to (or can compute) $M^\pi_t$ for each $\pi \in \Pi$. Consider the following algorithm:

\frameit{ 
			{\bf SCRiBLe while Learning the  Predictable Process} \\
			Input: $\eta > 0$, $\vartheta$-self-concordant $\cR$. Define $h_1 = \arg\min_{f \in \F} \cR(f) $. \\
			Initialize $q_1 \in \Delta(\Pi)$ as, $\forall \pi \in \Pi, q_1(\pi) = \frac{1}{\left|\Pi\right|}$\\
Set $M_1 = \sum_{\pi \in \Pi} q_1(\pi) M_1^\pi$\\
			At time $t=1$ to $T$\\
			\spce Let $ \{\ev{1}, \ldots, \ev{n}\} $ and $\{\lambda_1,\ldots, \lambda_n\}$ be the eigenvectors and eigenvalues of $\nabla^2 \cR(h_t)$.\\
		  	\spce Choose $i_t$ uniformly at random from $\{1,\ldots, n\}$ and $\varepsilon_t = \pm 1$ with probability $1/2$.\\
			\spce Predict $f_t = h_t + \varepsilon_t \lambda_{i_t}^{-1/2} \ev{i_t}$ and observe loss $\inner{f_t, x_t}$. \\
		  \spce Define $\tilde{x}_t := n \left(\inner{f_t, x_t-M_t} \right)\varepsilon_t \lambda_{i_t}^{1/2} \cdot \ev{i_t} + M_t$. \\
		  \spce Update 
		  $$\forall \pi \in \Pi,~ q_{t+1}(\pi) \propto q_{t}(\pi)\, e^{- (\ip{f_t}{x_t} - \ip{f_t}{M^\pi_t})^2} ~~\textrm{ and }~~ M_{t+1} = \sum_{\pi \in \Pi} q_{t+1}(\pi) M_{t+1}^\pi$$		  
		  $$ h_{t+1} = \arg\min_{h \in \F} \left[\eta\inner{h, \sum_{s=1}^t \tilde{x}_s + M_{t+1}} + \cR(h)\right].$$
}

The following lemma upper bounds the regret of this algorithm. The proof once again uses a regret bound in terms of the loss of the best arm \cite[Corollary 2.3]{PLG}. 
\begin{lemma}
	\label{lem:bandit1}
	Suppose that $\F,\X$ are contained in the $\ell_2$ ball of radius $1$. The expected regret of \emph{SCRiBLe while Learning the Predictable Process} is
	\begin{align}
		\En\left[ \sum_{t=1}^T \inner{f_t,x_t} - \sum_{t=1}^T \inner{f^*,x_t} \right] &\le \eta^{-1}\cR(f^*) + 2\eta n^2 \E{\sum_{t=1}^T  (\inner{f_t, x_t-M_t})^2} \label{eq:bandit1} \\
		&\leq \eta^{-1}\cR(f^*) + 13 \eta n^2 \left(\En\left[ \inf_{\pi \in \Pi}\sum_{t=1}^T  \|x_t-M^\pi_t\|^2\right] + \log \left| \Pi \right| \right)\notag \ .
	\end{align}	
\end{lemma}

\subsubsection{Partial Information about Predictable Process}
\label{subsubsubsec:partial_info_process}

Now let us consider the scenario where on each round we get to see $x_t \in \X$. However, we only see $M_t^{\pi_t}$ for a single $\pi_t \in \Pi$ we select on  round $t$. This scenario is especially useful in the stock investment example provided earlier. While $x_t$ the vector of losses for the stocks on each day can easily be obtained  at the end of the trading day, prediction processes might be provided as paid services by various companies. Therefore, we only get to access a limited number of forecasts on each day by paying for them. In this section, we provide an algorithm with corresponding regret bound for this case. 

\frameit{
\textbf{Optimistic MD with Learning the Predictable Processes with Partial Information}\\
Input: $\cR$ $1$-strongly convex w.r.t. $\|\cdot\|$, learning rate $\eta>0$\\
Initialize $ g_1 = \arg\min_{g} \cR(g)$ and initialize $q_1 \in \Delta(\Pi)$ as, $\forall \pi \in \Pi, q_1(\pi) = \frac{1}{\left|\Pi\right|}$\\
Sample $\pi_1 \sim q_1$ and set $f_1 = \argmin{f \in \F} \eta \inner{f,M^{\pi_{1}}_{1}} + D_\cR(f,g_{1})$\\
At $t=1,\ldots,T$, predict $f_t$ and :\\
\spce Update $q_t$ using SCRiBLe for multi-armed bandit with loss of arm $\pi_t$ : $\norm{M^{\pi_t}_t - x_t}^2_*   $  \\
\spce\spce and step-size $1/32 |\Pi|^2$.\\
\spce$\textrm{Sample }\pi_{t+1} \sim q_{t+1} \textrm{ and observe }M_{t+1}^{\pi_{t+1}}$\\
\spce Update
\begin{align*}
	g_{t+1} &= \argmin{g \in \F} \eta \inner{g, x_t} + D_\cR(g, g_t) \\
	f_{t+1} &= \argmin{f \in \F} \eta \inner{f,M^{\pi_{t+1}}_{t+1}} + D_\cR(f,g_{t+1})
\end{align*}
}

Due to the limited information about the predictable processes, the proofs of Lemmas~\ref{lem:two-step-MD-learn-partial} and \ref{lem:bandit3} below rely on an improved regret bound for the multiarmed bandit, an analogue of \cite[Corollary 2.3]{PLG}. Such a  bound is proved in Lemma~\ref{lem:non-stoch-multiarmed} in Section~\ref{sec:small_loss}.

\begin{lemma}
	\label{lem:two-step-MD-learn-partial}
	Let $\F$ be a convex set in a Banach space $\cB$ and $\cX$ be a convex set in the dual space $\cB^*$, both contained in unit balls. Let $\cR:\cB\mapsto\reals$ be a $1$-strongly convex function on $\F$ with respect to some norm $\|\cdot\|$. For any strategy of Nature, the Optimistic MD with Learning the Predictable Processes with Partial Information Algorithm yields, for any $f^*\in\F$,
\begin{align}
\E{\sum_{t=1}^T \inner{f_t, x_t}} - \sum_{t=1}^T \inner{f^*, x_t} & \leq  \eta^{-1}R_{\max}^2 +  \frac{\eta}{2}  \E{\sum_{t=1}^T \|x_t-M^{\pi_t}_t\|_*^2 } \label{eq:two-step-MD-learn-partial}\\
& \leq  \eta^{-1}R_{\max}^2 +  \eta \left(\En\inf_{\pi \in \Pi} \sum_{t=1}^T \|x_t-M^{\pi}_t\|_*^2 + 32  |\Pi|^3 \log(T |\Pi|) \right) \notag
\end{align}
	where $R_{\max}^2 = \max_{f\in\F} \cR(f) - \min_{f\in\F} \cR(f)$. 
\end{lemma}

\subsubsection{Partial Information about both Loss and Predictable Process}
\label{subsubsubsec:partial_info_loss_and_process}

In the third partial information variant, we consider the setting where at time $t$ we only observe loss $\inner{f_t,x_t}$ we suffer at the time step (and not entire $x_t$) and also only $M^{\pi_t}_t$ corresponding to the predictable process of  $\pi_t \in \Pi$ we select at time $t$. This is a blend of the two partial-information settings considered earlier.

\frameit{ 
			{\bf SCRiBLe for Learning the Predictable Process with Partial Feedback} \\
			Input: $\eta > 0$, $\vartheta$-self-concordant $\cR$. Define $h_1 = \arg\min_{f \in \F} \cR(f) $. \\
			Initialize $q_1 \in \Delta(\Pi)$ as, $\forall \pi \in \Pi, q_1(\pi) = \frac{1}{\left|\Pi\right|}$\\
Draw $\pi_1 \sim q_1$\\
			At time $t=1$ to $T$\\
			\spce Let $ \{\ev{1}, \ldots, \ev{n}\} $ and $\{\lambda_1,\ldots, \lambda_n\}$ be the eigenvectors and eigenvalues of $\nabla^2 \cR(h_t)$.\\
		  	\spce Choose $i_t$ uniformly at random from $\{1,\ldots, n\}$ and $\varepsilon_t = \pm 1$ with probability $1/2$.\\
			\spce Predict $f_t = h_t + \varepsilon_t \lambda_{i_t}^{-1/2} \ev{i_t}$ and observe loss $\inner{f_t, x_t}$. \\
		  \spce Define $\tilde{x}_t := n \left(\inner{f_t, x_t-M^{\pi_t}_t} \right)\varepsilon_t \lambda_{i_t}^{1/2} \cdot \ev{i_t} + M^{\pi_t}_t$. \\
		  \spce Update  $q_t$ using SCRiBLe for multi-armed bandit with loss \\
		  \spce\spce\spce of arm $\pi_t \in \Pi$: $(\ip{f_t}{x_t} - \ip{f_t}{M_t^{\pi_t}})^2   $ and step size $1/32 |\Pi|^2$.\\
		  \spce Draw $\pi_{t+1} \sim q_{t+1}$ and update
		  $$ h_{t+1} = \arg\min_{h \in \F} \left[\eta\inner{h, \sum_{s=1}^t \tilde{x}_s + M^{\pi_{t+1}}_{t+1}} + \cR(h)\right].$$
}

\begin{lemma}
	\label{lem:bandit3}
	Suppose that $\F,\X$ are contained in the $\ell_2$ ball of radius $1$. The expected regret of \emph{SCRiBLe for Learning the Predictable Process with Partial Feedback} is
	\begin{align}
		\En\left[ \sum_{t=1}^T \inner{f_t,x_t} - \sum_{t=1}^T \inner{f^*,x_t} \right] & \le \eta^{-1}\cR(f^*) + 2\eta n^2 \E{\sum_{t=1}^T  (\inner{f_t, x_t-M^{\pi_t}_t})^2} \label{eq:bandit3}\\
		&\leq \eta^{-1}\cR(f^*) + 4\eta n^2 \left(  \E{ \inf_{\pi \in \Pi} \sum_{t=1}^T  \norm{x_t-M^{\pi}_t}^2} + 32 |\Pi|^3 \log(T |\Pi|) \right) \notag \ .
	\end{align}
\end{lemma}

%% file: fpl.tex
\section{Randomized Methods and the Follow the Perturbed Leader Algorithm}
\label{sec:fpl-methods}

In this section we are back in the setting of Section~\ref{sec:fullinfo}, where the single process $M_t$ can be calculated by the learner. We show that randomized methods of the Follow the Perturbed Leader (FPL) style \cite{KalVem05,PLG} can also enjoy better bounds for predictable sequences. For convenience, we suppose $\F\subset \reals^d$ is a unit ball in some norm $\|\cdot\|$, and $\X$ is a unit ball in the dual norm $\|\cdot\|_*$ \ . 

The central object in the algorithmic development of \cite{RakShaSri12arxiv} is the notion of a relaxation. We now present this notion in the context of a \emph{constrained} adversary \cite{RakSriTew11nips} in order to develop randomized methods that attain bounds in terms of the sizes $\sigma_t$ of deviations from the trend $M_t$. The downside of the methods we present in this section is that individual deviations $\sigma_t$ need to be known in advance by the learner. We believe that this requirement can be relaxed, and this will be added in the full version of this paper.

A \emph{relaxation} $\mbf{Rel}$ is a sequence of functions $\Relax{T}{\F}{x_{1},\ldots,x_{t}}$ for each $t\in[T]$. We shall use the notation $\Rel{T}{\F}$ for $\Relax{T}{\F}{\{\}}$. For the problem of a constrained sequence, with constraints given by the sequence of $C_1,\ldots,C_T$ (see Eq.~\eqref{eq:def_constraint}) a relaxation will be called {\em admissible} if for any $x_1,\ldots,x_T \in \X$, 
\begin{align}
	\label{eq:relax_admissibility}
	\Relax{T}{\F}{x_{1},\ldots,x_{t}} \ge \inf_{q \in \Delta(\F)} \sup_{x \in C_{t+1}(x_1,\ldots,x_{t})} \Big\{ ~ \En_{f\sim q} \inner{f, x} + \Relax{T}{\F}{x_{1},\ldots,x_{t},x}\Big\}
\end{align}
for all $t\in[T-1]$, and 
$$ \Relax{T}{\F}{x_{1},\ldots,x_{T}} \ge - \inf_{f \in \F} \sum_{t=1}^T \inner{f, x_t} .$$ 
If $C_{t+1}(x_1,\ldots,x_{t}) = \cX$ for all $t\in[T]$, we recover the setting of an unconstrained adversary studied in \cite{RakShaSri12arxiv}. 

Any choice $q$ that ensures \eqref{eq:relax_admissibility} for an admissible relaxation   guarantees (irrespective of the strategy of the adversary) that
\begin{align}
	\label{eq:sum_cond_exp_bdd_by_relax}
	\sum_{t=1}^T \En_{f_t\sim q_t} \inner{f_t, x_t} - \inf_{f\in\F} \sum_{t=1}^T \inner{f, x_t} \leq \Rel{T}{\F} \ ,
\end{align}
a fact that is easy to prove. It is shown in \cite{RakShaSri12arxiv} that for many problems of interest, when searching for a computationally feasible relaxation, one may start with the conditional sequential Rademacher complexity and find a computationally attractive upper bound. For the case of constrained adversaries, this complexity becomes (for the case of $\F$ being a unit ball) 
\begin{align}
	\sup_{x_{t+1}\in C_{t+1}(x_{1:t})}\En_{\epsilon_{t+1}} \ldots \sup_{x_T\in C_T(x_{1:T-1})}\En_{\epsilon_T} \left\| 2\sum_{s=t+1}^T \epsilon_s (x_s-M_s(x_{1:s-1})) - \sum_{s=1}^t x_s\right\|_* 
\end{align}
which can be re-written as
\begin{align}
	\sup_{z_{t+1}: \|z_{t+1}\|_*\leq \sigma_{t+1}}\En_{\epsilon_{t+1}}\ldots \sup_{z_T: \|z_T\|_*\leq \sigma_T}\En_{\epsilon_T} \left\| 2\sum_{s=t+1}^T \epsilon_s z_s - \sum_{s=1}^t x_s\right\|_* 
\end{align}
Here, one may think of the adversary as choosing the $z_{t}$'s as small deviations from the predictable process $M_t$. The following step is a key idea: since the computation of the interleaved supremum and expectations is difficult, we might be able to come up with an almost-as-difficult distribution and draw $z_t$'s i.i.d. The following is an assumption that is easily verified for many  symmetric distributions \cite{RakShaSri12arxiv}.
\begin{assumption}
	\label{asm:fpl-linear}
	For every $t\in [T]$, there exists a distribution $D_t$ and constant $C \ge 2$ such that for any $w\in \reals^d$ 
	\begin{align}
		\label{eq:assumption_fpl_linear_simpler}
		\sup_{z: \|z\|_*\leq \sigma_t} \Eunderone{\epsilon} \norm{ w + 2\epsilon z }_* \le \Eunderone{z \sim D_t }\Eunderone{\epsilon} \norm{ w+  C\epsilon z  }_* 
	\end{align}	
	and $\En_{z\sim D_t}\|z\|^2_* \leq \sigma_t^2$ for any $t$.
\end{assumption}

To satisfy this assumption, one may simply take one of the distributions in \cite{RakShaSri12arxiv} for the unconstrained case, and scale it by $\sigma_t$.

\begin{lemma}
	\label{lem:fpl}
	For the distributions $D_1,\ldots,D_T$ satisfying Assumption~\ref{asm:fpl-linear}, the relaxation 			
	\begin{align}
		\label{eq:fplrel}
		\Relax{T}{\F}{x_1,\ldots,x_t} = \Eunderone{z_{t+1}\sim D_{t+1},\ldots z_T \sim D_{T}}\En_{\epsilon}   \left\|C \sum_{i=t+1}^T \epsilon_i z_i - \sum_{i=1}^t x_i \right\|_*
	\end{align}
	 is admissible and a randomized strategy that ensures admissibility is given by: at time $t$, draw $z_{t+1},\ldots,z_{T}$ and  Rademacher random variables $\epsilon=(\epsilon_{t+1},\ldots,\epsilon_T)$, and then define
\begin{align}
	\label{eq:def_general_fpl}
	f_t = \argmin{g \in \F} \sup_{x_t\in C_t(x_{1:t-1})} \left\{\inner{g,x_t} + \norm{ C \sum_{i=t+1}^T \epsilon_i z_i  - \sum_{i=1}^{t-1} x_i - x_t }_* \right\}
\end{align}
	The expected regret for the method is bounded by the classical Rademacher complexity
$$
\En{\Reg_T} \le C\ \En_{z_{1:T}} \En_{\epsilon} \left\|\sum_{t=1}^T \epsilon_t z_t\right\|_*
$$
where each random variable $z_t$ has distribution $D_t$. For any smooth norm, the expected regret can be further upper bounded by $O\left(\sqrt{\sum_{t=1}^T \sigma_t^2}\right)$.

\end{lemma}

Let us define the random vector
$$R_t ~:=~  \sum_{i=1}^{t-1} x_i- C \sum_{i=t+1}^T \epsilon_i z_i + M_t$$
where the first sum is the cumulative cost vector, the second  sum may be viewed as a random perturbation of the cumulative cost, and the final term is simply the predictable process at time $t$. 
We may rewrite \eqref{eq:def_general_fpl} as
\begin{align}
	f_t & = \argmin{f \in \F} \sup_{x_t\in C_t(x_{1:t-1})} \Big\{  \inner{f,x_t} + \norm{R_t+x_t - M_t}_* \Big\}\notag \\
	& = \argmin{f \in \F} \sup_{z : \norm{z}_* \le \sigma_t } \Big\{  \inner{f,z + M_t} + \norm{R_t+z}_* \Big\} \label{eq:def_general_fpl_with_R}
\end{align}
This is a general form of the randomized method for online linear optimization. As shown in \cite{RakShaSri12arxiv}, this form in fact reduces to the more familiar form of the FPL update in certain cases.

\subsection{Randomized Algorithm for the $\ell_1/\ell_\infty$ Case}

We now show that for the case of $\F$ being an $\ell_1$ ball and $\cX$ being an $\ell_\infty$ ball, the solution in \eqref{eq:def_general_fpl_with_R} takes on a simpler form. In particular, for $M_t=0$ the solution is simply an indicator on the maximum coordinate of $R_t$, which is precisely the Follow the Perturbed Leader solution.

% The above lemma essentially shows that the update in Equation \eqref{eq:def_general_fpl_with_R} can be rewritten as
% \begin{align*}
% f_t & = \argmin{f : \norm{f}_1 \le 1} \left\{\sigma_t \sum_{i\ne j^*_t} |f[i]| +  \sigma_t f[j^*_t] \sign(R_t[j^*_t]) + \ip{f}{M_t} \right\}
% \end{align*}
% where $R_t =  \sum_{i=1}^{t-1} x_i- C \sum_{i=t+1}^T \epsilon_i z_i + M_t$ and $j^*_t = \argmax{j \in [d]} |R_t[j]|$. The theorem below gives a simple form of follow the perturbed leader update for the $\ell_1/\ell_\infty$ case of online learning with predictable processes. 

\begin{theorem}
	\label{thm:fpl_update}
For the distributions $D_1,\ldots,D_T$ satisfying Assumption~\ref{asm:fpl-linear}, consider the randomized strategy that at time $t$, draws $z_{t+1},\ldots,z_T$ from $D_{t+1},\ldots,D_T$ respectively and Rademacher random variables $\epsilon=(\epsilon_{t+1},\ldots,\epsilon_T)$, and then outputs
\begin{align}\label{eq:fplup}
f_t = \left\{ \begin{array}{cl}
- \sign(M_t[i^*_t]) e_{i^*_t} & \textrm{if }\sigma_t - |M_t[i^*_t]| <  - \left|\sigma_t\ \sign(R_t[j^*]) + M_t[j^*_t]\right| \\
- \sign(\sigma_t R_t[j^*_t] + M_t[j^*_t]) e_{j^*_t}  & \textrm{otherwise}
\end{array}
\right.
\end{align}
where $R_t =  \sum_{i=1}^{t-1} x_i- C \sum_{i=t+1}^T \epsilon_i z_i + M_t$, $j^*_t = \argmax{j \in [d]} |R_t[j]|$ and $i^*_t = \argmax{i \in [d]} |M_t[i]|$. The expected regret is bounded as :
$$
\E{\Reg_T} \le C\ \En_{z_{1:T}} \En_{\epsilon} \left\|\sum_{t=1}^T \epsilon_t z_t\right\|_* + 4\ \sum_{t=1}^T  \P\left( {\mathcal E}_t^c \right) \ .
$$
\end{theorem}

% !TEX root =  paper.tex

\subsection{Randomized Algorithm for the Simplex}

Given an algorithm for regret minimization over the probability simplex (as in the case of experts), through a standard argument one also obtains an algorithm for the $\ell_1$ ball by doubling the number of coordinates. We now show that the randomized method for the $\ell_1$ ball, developed in the previous section, can be used to solve the problem over the probability simplex, a converse implication. Specifically, we have the following corollary:
\begin{corollary}
	\label{cor:fpl_simplex_update}
	For the distributions $D_1,\ldots,D_T$ satisfying Assumption~\ref{asm:fpl-linear}, consider the randomized strategy that at time $t$, draws $z_{t+1},\ldots,z_T$ from $D_{t+1},\ldots,D_T$ respectively and Rademacher random variables $\epsilon=(\epsilon_{t+1},\ldots,\epsilon_T)$, and then outputs 
	\begin{align}\label{eq:fplsimplexup}
	f_t = \left\{ \begin{array}{cl}
	 e_{i^*_t} & \textrm{if } 2 \sigma_t   < M[j^*_t] - M_t[i^*_t] \\
	 e_{j^*_t}  & \textrm{otherwise}
	\end{array}
	\right.
	\end{align}
	where $R_t =  \sum_{i=1}^{t-1} x_i- C \sum_{i=t+1}^T \epsilon_i z_i + M_t$, $j^*_t = \argmin{j \in [d]} R_t[j]$ and $i^*_t = \argmin{i \in [d]} M_t[i]$. The expected regret is bounded as:
	$$
	\E{\Reg_T} \le C\ \En_{z_{1:T}} \En_{\epsilon} \left\|\sum_{t=1}^T \epsilon_t z_t\right\|_* + 4\ \sum_{t=1}^T  \P\left( {\mathcal E}_t^c \right) \ .
	$$
\end{corollary}

When the predictable sequence $M_t$ is zero, the algorithm reduces to $f_t = e_{j^*_t}$ with 
$$j^*_t = \argmax{j \in [d]} \left| \sum_{i=1}^{t-1} x_i- C \sum_{i=t+1}^T \epsilon_i z_i \right|$$ 
which can be recognized as a Follow the Perturbed Leader type update with $\sum_{i=1}^{t-1} x_i$ being the cumulative loss and $\sum_{i=t+1}^T \epsilon_i z_i$ being a random perturbation.

%% file: other.tex
\section{Other Examples}

We now provide a couple of examples and sketch directions for further research.

\subsection{Delayed Feedback}
\label{sec:delayed_feedback}

As an example, consider the setting where the information given to the player at round $t$ consists of two parts: the bandit feedback $\inner{f_t,x_t}$ about the cost of the chosen action, as well as full information about the past move $x_{t-k}$. For $t>k$, let $M_t = M_t (I_1,\ldots,I_{t-1}) = \frac{1}{t-k-1}\sum_{s=1}^{t-k-1} x_s$.  Then $$\|M_t-M'_t\|^2 = \left\| \frac{1}{t-k-1}\sum_{s=1}^{t-k-1} x_s - \frac{1}{t-1}\sum_{s=1}^{t-1} x_s \right\|^2 \leq \left\| \frac{k}{(t-1)(t-k-1)}\sum_{s=1}^{t-k-1} x_s - \frac{1}{t-1}\sum_{s=t-k}^{t-1} x_s \right\|^2 \leq \frac{4k^2}{(t-1)^2},$$
where $M'_t = \frac{1}{t-1}\sum_{s=1}^{t-1} x_s$ is the full information statistic. It is immediate from Lemma~\ref{lem:bandit} that the expected regret of the algorithm is 
	\begin{align*}
		\En\left[ \sum_{t=1}^T \inner{f_t,x_t} - \sum_{t=1}^T \inner{f^*,x_t} \right]
		&\leq \eta^{-1}\cR(f^*) + 4\eta n^2\sum_{t=1}^T \En\left[\|x_t-M'_t\|^2\right] + 32\eta n^2 k^2
	\end{align*}
This simple argument shows that variance-type bounds are immediate in bandit problems with delayed full information feedback. 

\subsection{I.I.D. Data}

Consider the case of i.i.d. sequence $x_1,\ldots,x_T$ drawn from an unknown distribution with mean $\mu\in \reals^d$. Let us first discuss the full-information model. Consider the bound of either Lemma~\ref{lem:self_conc_full_info} or Lemma~\ref{lem:two-step-MD} for $M_t = \frac{1}{t-1}\sum_{s=1}^{t-1} x_s$. For simplicity, let $\|\cdot\|$ be the Euclidean norm (the argument works with any smooth norm). We may write
$$\|x_t-M_t\|^2 \leq \|x_t - \mu\|^2 + \|M_t - \mu\|^2 + 2\inner{x_t - \mu, M_t - \mu} \ .$$
Taking the expectation over i.i.d. data, the first term in the above bound is variance $\sigma^2$ of the distribution under the given norm, while the third term disappears under the expectation. For the second term, we perform exactly the same quadratic expansion and obtain
$$\En\|M_t - \mu\|^2\leq \frac{1}{(t-1)^2}\sum_{s=1}^{t-1} \En\|x_t-\mu\|^2 \leq \frac{\sigma^2}{t-1}$$
and thus
$$\sum_{t=1}^T\En\|x_t-M_t\|^2 \leq T\sigma^2 + \sigma^2(\log T +1)$$
Coupled with the full-information results of Lemma~\ref{lem:self_conc_full_info} or Lemma~\ref{lem:two-step-MD}, we obtain an
$\tilde{O}(\sigma \sqrt{T})$
bound on regret, implying the natural transition from the noisy to deterministically predictable case as the noise level goes to zero.

The same argument works for the case of bandit information, given that $M_t$ can be constructed to estimate $M'_t$ well (e.g. using the arguments of \cite{hazan2009better}).

%% file: small_losses.tex
\section{Auxiliary Results: Improved Bounds for Small Losses}
\label{sec:small_loss}

While the regret bound for the original SCRiBLe algorithm follows immediately from the more general Lemma~\ref{lem:bandit}, we now state an alternative bound for SCRiBLe in terms of the loss of the optimal decision. The bound holds under the assumption of positivity on the losses. Lemma~\ref{lem:banditlstar} is of independent interest and will be used as a building block for the analogous result for the multi-armed bandit in Lemma~\ref{lem:non-stoch-multiarmed}. Such bounds in terms of the loss of the best arm are attractive, as they give tighter results  whenever the loss of the optimal decision is small. Thanks to this property, Lemma~\ref{lem:non-stoch-multiarmed} is used in Section~\ref{sec:learning} in order to obtain bounds in terms of predictable process performance.
\begin{lemma}
	\label{lem:banditlstar}
Consider the case when $\cR$ is a self-concordant barrier over $\F$ and sets $\F$ and $\X$ are such that each $\ip{f}{x} \in [0,s]$. Then for the SCRiBLe algorithm, for any choice of step size $\eta < 1/(2 s n^2)$, we have the bound
	\begin{align*}
\E{\sum_{t=1}^T  \inner{f_t, x_t}} \le \frac{1}{1 - (2 s n^2) \eta } \left(\sum_{t=1}^T \inner{f^*,x_t} +  \eta^{-1}\cR(f^*)  \right)	\end{align*}
\end{lemma}

We now state and prove a bound in terms of the loss of the best arm  for the case of non-stochastic multiarmed bandits. Such a bound is  interesting in its own right and, to the best of our knowledge, it does not appear in the literature.\footnote{The bound of \cite{AueCesFreSch03nonstochastic} is in terms of maximal gains, which is very different from a bound in terms of minimal loss. To the best of our knowledge, the trick of redefining losses as negative gains does not work here.} Our approach is to use SCRiBLe with a self-concordant barrier for the probability simplex, coupled with the bound of Lemma~\ref{lem:banditlstar}. (We were not able to make this result work with the entropy function, even with the local norm bounds).

Suppose that Nature plays a sequence $x_1,\ldots,x_T \in [0,s]^d$. On each round, we chose an arm $j_t$ and observe $\inner{e_{j_t}, x_t}$. 

\frameit{
			{\bf SCRiBLe for multi-armed Bandit}  \cite{abernethy2012interior,abernethy2008competing}\\
			Input: $\eta > 0$. Let $\cR(f) = - \sum_{i=1}^{d-1} \log(f[i]) - \log(1 - \sum_{i=1}^{d-1} f[i])$\\
			Initialize $q_1$ with uniform distribution over arms. Let $h_1 = q_1[1:d-1]$\\
			At time $t=1$ to $T$\\
		    \spce Let $ \{\ev{1}, \ldots, \ev{d-1}\} $ and $\{\lambda_1,\ldots, \lambda_{d-1}\}$ be the eigenvectors and eigenvalues of $\nabla^2 \cR(h_t)$.\\
		    \spce Choose $i_t$ uniformly at random from $\{1,\ldots, [d-1]\}$ and $\varepsilon_t = \pm 1$ with probability $1/2$.\\
			\spce Set $f_t = h_t + \varepsilon_t \lambda_{i_t}^{-1/2} \ev{i_t}$  and $q_t = (f_t,1-\sum_{i=1}^{d-1}f_t[i])$. \\
			\spce Draw arm $j_t \sim q_t$ and suffer loss $\ip{e_{j_t}}{x_t}$.\\
		  	\spce Define $\tilde{x}_t := d \left(\inner{e_{j_t}, x_t} \right)\varepsilon_t \lambda_{i_t}^{1/2} \cdot \ev{i_t}$. \\
		  	\spce Update $$ h_{t+1} = \arg\min_{h \in \reals^{d-1}} \left[\eta\inner{h, \sum_{s=1}^t \tilde{x}_s} + \cR(h)\right].$$
}

\begin{lemma}\label{lem:non-stoch-multiarmed}
	Suppose $x_1,\ldots,x_T \in [0,s]^d$. For any $\eta<1/(4sd^2)$ the expected regret of the SCRiBLe for multi-armed Bandit algorithm is bounded as :
	$$\En \left\{ \sum_{t=1}^T \inner{e_{j_t}, x_t} \right\} \leq \frac{1}{1-4 \eta s d^2}\left(\inf_{j\in[d]}\sum_{t=1}^T \inner{e_{j}, x_t} + d \eta^{-1}\log (d T)\right)$$
\end{lemma}

%% file: dbltrick.tex
% !TEX root =  paper.tex

\section{Standard Doubling Trick}
\label{sec:dbltrick}

For completeness, we now describe a more or loss standard doubling trick, extending it to the case of partial information. Let ${\mathcal I}$ stand for some information space such that the algorithm receives $I_t\in{\mathcal I}$ at time $t$, as described in the introduction. Let $\Psi:\cup_s({\mathcal I}\times \F)^s \mapsto \reals$ be a (deterministic) function defined for any contiguous time interval of any size $s\in [T]$. By the definition, $\Psi(I_r,\ldots,I_t,f_r,\ldots,f_t)$ 
is computable by the algorithm after the $t$-th step, for any $r\leq t$. We make the following monotonicity assumption on $\Psi$: for any $I_1,\ldots,I_{t} \in {\mathcal I}$ and any $f_1,\ldots,f_t \in \F$, $\Psi(I_{1:t-1},f_{1:t-1}) \le \Psi(I_{1:t}, f_{1:t})$ and $\Psi(I_{2:t},f_{2:t}) \le \Psi(I_{1:t}, f_{1:t})$.

\begin{lemma}\label{lem:double}
	
	Suppose we have a randomized algorithm that takes a fixed $\eta$ as input and for some constant $A$ without a priori knowledge of $\tau$, for any $\tau > 0$, guarantees expected regret of the form
	$$
	\E{\sum_{t=1}^\tau \loss(f_t,x_t) - \inf_{f\in\F} \sum_{t=1}^\tau \loss(f,x_t)} \leq A\eta^{-1} + \eta\E{\Psi(I_{1:\tau},f_{1:\tau})}
	$$
	where $\Psi$ satisfies the above stated requirements. Then using this algorithm as a black-box for any $T>0$, we can provide a randomized algorithm with a regret bound 
\begin{align*}
\E{\sum_{t=1}^T \loss(f_t,x_t) - \inf_{f\in\F} \sum_{t=1}^T \loss(f,x_t)} \leq 16 \sqrt{A \E{\Psi(I_{1:T},f_{1:T})}}
\end{align*}
\end{lemma}

\begin{proof}
The prediction problem is broken into phases, with a constant learning rate $\eta_i=\eta_0 2^{-i}$ throughout the $i$-th phase, for some $\eta_0>0$. Define for $i\geq 1$
$$s_{i+1} = \min \left\{\tau: \eta_{i} \Psi(I_{s_{i}:\tau},f_{s_{i}:\tau}) > A\eta_{i}^{-1}\right\}$$
to be the start of the phase $i+1$, and $s_1 = 1$. Let $N$ be the last phase of the game and let $s_{N+1}= T+1$. Without loss of generality, assume $N>1$ (for, otherwise regret is at most $4A/\eta_0$). Then
\begin{align*}
	\E{\sum_{t=1}^T \loss(f_t,x_t) - \inf_{f\in\F} \sum_{t=1}^T \loss(f,x_t)} &\leq \E{\sum_{k=1}^{N} \Es{f_{s_k:s_{k+1}-1}}{ \sum_{t=s_{k}}^{s_{k+1}-1} \loss(f_t,x_t) - \inf_{f\in\F}\sum_{t=s_{k}}^{s_{k+1}-1} \loss(f,x_t) }} \notag\\
	&\leq \E{\sum_{k=1}^{N} \left(  A\eta_k^{-1} + \eta_k \Es{f_{s_k:s_{k+1}-1}}{ \Psi(I_{s_{k}:s_{k+1}-1},f_{s_{k}:s_{k+1}-1})} \right) }\\
	&\leq 2\E{\sum_{k=1}^{N} A\eta_k^{-1} }
\end{align*}
where the last inequality follows because $\eta_k  \Psi(I_{s_k:s_{k+1}-1},f_{s_k:s_{k+1}-1}) \leq A\eta_k^{-1}$ 
within each phase. Also observe that
$$ \eta_{N-1} \Psi(I_{s_{N-1}:s_{N}},f_{s_{N-1}:s_{N}}) > A\eta_{N-1}^{-1},$$ 
which implies
$$\eta_0^{-1} 2^{N}=\eta_{N}^{-1} = 2\eta_{N-1}^{-1} < 2\sqrt{\frac{\Psi(I_{s_{N-1}:s_{N}},f_{s_{N-1}:s_{N}})}{A}}\leq 2\sqrt{\frac{\Psi(I_{1:T},f_{1:T})}{A}} $$
by the monotonicity assumption. Hence, regret is upper bounded by
$$2\sum_{k=1}^{N} A\eta_k^{-1} = 2A\eta_0^{-1} 2^{N}\sum_{k=1}^{N} 2^{k-N}\leq 4A\eta_0^{-1} 2^{N} \leq 8\sqrt{A\ \Psi(I_{1:T},f_{1:T})}$$
Putting the arguments together,
\begin{align*}
	\E{\sum_{t=1}^T \loss(f_t,x_t) - \inf_{f\in\F} \sum_{t=1}^T \loss(f,x_t)} &\leq 8 \E{\sqrt{A\ \Psi(I_{1:T},f_{1:T})}} \le 8 \sqrt{A\ \E{\Psi(I_{1:T},f_{1:T})}} 
\end{align*}

Now, observe that the rule for stopping the phase can only be calculated \emph{after} the first time step of the new phase. The easiest way to deal with this is to throw out $N$ time periods and suffer an additional regret of $s N$ (losses are bounded by $s$). Using $\eta_0 = 4 A/s$ this leads to additional factor of $s N \le s 2^N =  4A\eta_0^{-1} 2^{N} \leq 8\sqrt{A\ \Psi(I_{1:T},f_{1:T})}$, which is a gross over-bound. In conclusion, the overall bound on regret is
$$
\E{\sum_{t=1}^T \loss(f_t,x_t) - \inf_{f\in\F} \sum_{t=1}^T \loss(f,x_t)} \leq 16 \sqrt{A \E{\Psi(I_{1:T},f_{1:T})}}~.
$$  
\end{proof}

We remark that while the algorithm may or may not start each new phase from a cold start (that is, forget about what has been learned), the functions $M_t$ may still contain information about all the past moves of Nature.

With this doubling trick, for any of the full information bounds presented in the paper (for instance Lemmas  \ref{lem:self_conc_full_info}, \ref{lem:two-step-MD}, \ref{lem:exp-weights-local} and \ref{lem:two-step-MD-learn}) we can directly get an algorithm that enjoys a regret bound that is a factor at most $8$ from the bound with optimal choice of $\eta$. 

For Lemmas \ref{lem:bandit}, \ref{lem:bandit1}, \ref{lem:two-step-MD-learn-partial} and \ref{lem:bandit3}, we need to apply the doubling trick to an intermediate quantity, as the final bound is given in terms of quantities not computable by the algorithm. Specifically, the doubling trick needs to be applied to  Equations \eqref{eq:bandit}, \eqref{eq:bandit1}, \eqref{eq:two-step-MD-learn-partial} and \eqref{eq:bandit3}, respectively, in order to get bounds that are within a factor $8$ from the bounds obtained by optimizing $\eta$ in the corresponding equations. We can then upper these computable quantities by corresponding unobserved quantities as is done in these lemmas. To see this more clearly let us demonstrate this on the example of Lemma~\ref{lem:bandit3}. By Equation~\eqref{eq:bandit3}, we have that
$$
\En\left[ \sum_{t=1}^T \inner{f_t,x_t} - \sum_{t=1}^T \inner{f^*,x_t} \right]  \le \eta^{-1}\cR(f^*) + 2\eta n^2 \E{\sum_{t=1}^T  (\inner{f_t, x_t-M^{\pi_t}_t})^2} 
$$
Now note that $(\inner{f_t, x_t-M^{\pi_t}_t})^2$ is a quantity computable by the algorithm at each round. Also note that $2 \eta n^2 \sum_{t=1}^T  (\inner{f_t, x_t-M^{\pi_t}_t})^2$ satisfies the condition on $\Psi$ required by Lemma \ref{lem:double}, as the sum of squares is monotonic. Hence using the lemma we can conclude that
\begin{align}\label{eq:dblint}
\En\left[ \sum_{t=1}^T \inner{f_t,x_t} - \sum_{t=1}^T \inner{f^*,x_t} \right]  \le 16 \sqrt{2 n^2 \cR(f^*)   \E{\sum_{t=1}^T  (\inner{f_t, x_t-M^{\pi_t}_t})^2} }
\end{align}
The following steps in Lemma \ref{lem:bandit3} (see proof in the Appendix) imply that 
$$
\E{\sum_{t=1}^T  (\inner{f_t, x_t-M^{\pi_t}_t})^2} \le 2 \left(  \E{ \inf_{\pi \in \Pi} \sum_{t=1}^T  \norm{x_t-\Mbar^{\pi}_t}^2} + 32 |\Pi|^3 \log(T |\Pi|) \right) 
$$
Plugging the above in Equation \ref{eq:dblint} we can conclude that 
$$
\En\left[ \sum_{t=1}^T \inner{f_t,x_t} - \sum_{t=1}^T \inner{f^*,x_t} \right]  \le 16 \sqrt{4 n^2 \cR(f^*)  
\left(  \E{ \inf_{\pi \in \Pi} \sum_{t=1}^T  \norm{x_t-\Mbar^{\pi}_t}^2} + 32 |\Pi|^3 \log(T |\Pi|) \right) }
$$
This is exactly the inequality one would get if the final bound in Lemma \ref{lem:bandit3} is optimized for $\eta$, with an additional factor of $8$. With similar argument we can get the tight bounds for Lemmas \ref{lem:bandit}, \ref{lem:bandit1} and \ref{lem:two-step-MD-learn-partial} too, even though they are in the bandit setting.

%% file: appendix.tex
\begin{proof}[\textbf{Proof of Lemma~\ref{lem:self_conc_full_info}}]
	Define $g_{t+1} = \arg\min_{f\in\F}~ \eta\inner{f, \sum_{s=1}^t x_s} + \cR(f)$ to be the (unmodified) Follow the Regularized Leader. Observe that for any $f^*\in\F$,
	\begin{align}
		\label{eq:decomp}
		\sum_{t=1}^T\inner{f_t-f^*, x_t} = \sum_{t=1}^T\inner{f_t-g_{t+1},x_t-M_{t}} +\sum_{t=1}^T \inner{f_t-g_{t+1}, M_t} + \sum_{t=1}^T\inner{g_{t+1}-f^*, x_t}
	\end{align}
	
	We now prove by induction that
	$$\sum_{t=1}^\tau \inner{f_t-g_{t+1}, M_t} + \sum_{t=1}^\tau \inner{g_{t+1}, x_t} \leq \sum_{t=1}^\tau \inner{f^*, x_t}+\eta^{-1}\cR(f^*).$$
	The base case $\tau=1$ is immediate since $M_1=0$. For the purposes of induction, suppose that the above inequality holds for $\tau=T-1$. Using $f^*=f_T$ and adding $\inner{f_T-g_{T+1}, M_{T}} + \inner{g_{T+1}, x_T}$ to both sides,
	\begin{align*}
		\sum_{t=1}^{T} \inner{f_t-g_{t+1}, M_t} + \sum_{t=1}^{T}\inner{g_{t+1}, x_t} &\leq \sum_{t=1}^{T-1} \inner{f_T, x_t}+\eta^{-1}\cR(f_T) + \inner{f_T-g_{T+1}, M_{T}} + \inner{g_{T+1}, x_T} \\
		&\leq  \inner{f_T, \sum_{t=1}^{T-1} x_t + M_T}+\eta^{-1}\cR(f_T) - \inner{g_{T+1}, M_{T}} + \inner{g_{T+1}, x_T} \\
		&\leq  \inner{g_{T+1}, \sum_{t=1}^{T-1} x_t + M_T}+\eta^{-1}\cR(g_{T+1}) - \inner{g_{T+1}, M_{T}} + \inner{g_{T+1}, x_T} \\
		&\leq  \inner{g^*, \sum_{t=1}^{T} x_t}+\eta^{-1}\cR(g^*)
	\end{align*}
	by the optimality of $f_T$ and $g_{T+1}$. This concludes the inductive argument, and from Eq.~\eqref{eq:decomp} we obtain 
	\begin{align}
		\sum_{t=1}^T\inner{f_t-f^*, x_t} \leq \sum_{t=1}^T\inner{f_t-g_{t+1},x_t-M_{t}} + \eta^{-1} \cR(f^*)
	\end{align}

	Define the Newton decrement for $\Phi_t(f)\deq\eta\inner{f, \sum_{s=1}^t x_s + M_{t+1}} + \cR(f)$ as 
	$$ \lambda (f, \Phi_t) := \| \nabla \Phi_t (f) \|^*_{f} = \|\nabla^2 \Phi_t(f)^{-1} \nabla \Phi_t(f)\|_{f}.
	$$
	Since $\cR$ is self-concordant then so is $\Phi_t$, with their Hessians coinciding. The Newton decrement measures how far a point is from the global optimum. The following result can be found, for instance, in \cite{NemTod08}:
			For any self-concordant function $\tilde{\cR}$, whenever $\lambda(f,\tilde{\cR}) < 1/2$, we have
		$$\| f - \arg\min \tilde{\cR} \|_{f} \leq 2 \lambda(f,\tilde{\cR})$$
		where the local norm $\| \cdot \|_{f}$ is defined with respect to $\tilde{\cR}$, i.e. $\| g \|_{f} := \sqrt{g^\tr (\nabla^2 \tilde{\cR}(f) )g}$. 
	Applying this to $\Phi_t$ and using the fact that $\nabla\Phi_{t-1}(g_{t+1}) = \eta(M_{t}-x_{t})$,
		\begin{align}
			\label{eq:closeness}
			\|f_t - g_{t+1}\|_{f_t} = \|g_{t+1} - \arg\min \Phi_t \|_{f_t} \leq 2 \lambda(g_{t+1},\Phi_t) = 2 \eta \| M_t-x_t \|_{f_t}^*.
		\end{align}
		Hence,
		\begin{align*}
			\sum_{t=1}^T\inner{f_t-f^*, x_t} &\leq \sum_{t=1}^T \|f_t-g_{t+1}\|_{t} \|x_t-M_{t}\|^*_t + \eta^{-1} \cR(f^*)\\
			&\leq 2\eta\sum_{t=1}^T (\|x_t-M_{t}\|_{f_t}^*)^2 + \eta^{-1} \cR(f^*),
		\end{align*}
		which proves the statement.
\end{proof}

\begin{proof}[\textbf{Proof of Lemma~\ref{lem:two-step-MD}}]
For any $f^*\in\F$,
\begin{align}\label{eq:head}
	\inner{f_t-f^*, x_t} = \inner{f_t-g_{t+1},x_t-M_{t}} + \inner{f_t-g_{t+1}, M_t} + \inner{g_{t+1}-f^*, x_t}
\end{align}
First observe that 
\begin{align}
	\label{eq:part-one}
	\inner{f_t-g_{t+1},x_t-M_{t}} \le \norm{f_t-g_{t+1}} \norm{x_t - M_t}_* \le \frac{\eta}{2} \norm{x_t - M_t}_*^2 + \frac{1}{2 \eta} \norm{f_t - g_{t+1}}^2 \ .
\end{align}
On the other hand, any update of the form $a^* = \arg\min_{a\in A} \inner{a,x} + D_\cR(a,c)$ satisfies for any $d\in A$ (see e.g. \cite{beck2003mirror, lecturenotes08})
\begin{align}
	\label{eq:md-inequality}
	\inner{a^*-d,x} \leq D_\cR(d,c)- D_\cR(d,a^*)-D_\cR(a^*,c) \ .
\end{align}
This yields
\begin{align}
	\label{eq:part-three}
	\inner{f_t-g_{t+1}, M_t} \leq \frac{1}{\eta} \left(D_\cR(g_{t+1},g_t) - D_\cR(g_{t+1},f_t) - D_\cR(f_t,g_t)  \right) 
\end{align}
and
\begin{align}
	\label{eq:part-two}
	\inner{g_{t+1}-f^*, x_t} &\leq  \frac{1}{\eta}\left(D_\cR(f^*,g_{t})  - D_\cR(f^*,g_{t+1}) - D_\cR(g_{t+1},g_{t}) \right) . 
\end{align}
Using Equations \eqref{eq:part-one}, \eqref{eq:part-two} and \eqref{eq:part-three} in Equation \eqref{eq:head} we conclude that 
\begin{align*}
	\inner{f_t-f^*, x_t} 
	&\leq \frac{\eta}{2} \norm{x_t - M_t}_*^2 + \frac{1}{2 \eta} \norm{f_t - g_{t+1}}^2 \\
	& +\frac{1}{\eta} \left(D_\cR(g_{t+1},g_t) - D_\cR(g_{t+1},f_t) - D_\cR(f_t,g_t) \right) \\
	& + \frac{1}{\eta}\left(D_\cR(f^*,g_{t})  - D_\cR(f^*,g_{t+1}) - D_\cR(g_{t+1},g_{t})) \right)\\
	&\le \frac{\eta}{2} \norm{x_t - M_t}_*^2 + \frac{1}{2 \eta} \norm{f_t - g_{t+1}}^2 
 + \frac{1}{\eta}\left(D_\cR(f^*,g_{t})  - D_\cR(f^*,g_{t+1})  - D_\cR(g_{t+1},f_t) \right)
\end{align*}
By strong convexity of $\cR$, $D_\cR(g_{t+1},f_t) \ge \frac{1}{2} \norm{g_{t+1} - f_t}^2$ and thus
$$\inner{f_t-f^*, x_t} \leq \frac{\eta}{2} \norm{x_t - M_t}_*^2 + \frac{1}{\eta}\left(D_\cR(f^*,g_{t})  - D_\cR(f^*,g_{t+1})  \right)$$

Summing over $t=1,\ldots,T$ yields, for any $f^*\in\F$,
$$\sum_{t=1}^T \inner{f_t-f^*, x_t} \leq \frac{\eta}{2} \sum_{t=1}^T \|x_t-M_t\|_*^2 + \frac{ R_{\max}^2}{\eta}$$
where $R_{\max}^2 = \max_{f\in\F} \cR(f) - \min_{f\in\F} \cR(f)$. %Choosing $\eta = \frac{\sqrt{2}R_{\max}}{\sqrt{\sum_{t=1}^T \|x_t-M_t\|_*^2}}$ balances the two terms.
\end{proof}

\begin{proof}[\textbf{Proof of Lemma~\ref{lem:exp-weights-local}}]
The proof closely follows the proof of Lemma~\ref{lem:two-step-MD} and together with the technique of \cite{AbeRak09}. For the purposes of analysis, let $g_{t+1}$ be a projected point at every step (that is, normalized). Then we have the closed form solution for $f_{t}$ and $g_{t+1}$:
	$$g_{t+1}(i) = \frac{\exp\{-\eta\sum_{s=1}^t x_s(i)\}}{\sum_{j=1}^d\exp\{-\eta\sum_{s=1}^t x_s(j)\}} ~\mbox{and}~ f_t(i) = \frac{\exp\{-\eta\sum_{s=1}^{t-1} x_s(i) - \eta M_t(i)\}}{\sum_{j=1}^d\exp\{-\eta\sum_{s=1}^{t-1} x_s(j) - \eta M_t(j)\}}$$
Hence,
\begin{align}
	\label{eq:ratio-of-updates}
	\frac{g_{t+1}(i)}{f_t(i)} &= \frac{\exp\{-\eta\sum_{s=1}^t x_s(i)\}}{\exp\{-\eta\sum_{s=1}^{t-1} x_s(i) - \eta M_t(i)\}} \frac{\sum_{j=1}^d\exp\{-\eta\sum_{s=1}^{t-1} x_s(j) - \eta M_t(j)\}}{\sum_{j=1}^d\exp\{-\eta\sum_{s=1}^t x_s(j)\}} \notag\\
	&= \exp\{-\eta (x_t(i)-M_t(i))\}\frac{\sum_{j=1}^d\exp\{-\eta\sum_{s=1}^{t-1} x_s(j) - \eta M_t(j)\}}{\sum_{j=1}^d\exp\{-\eta\sum_{s=1}^t x_s(j)\}\exp\left\{-\eta(x_t(i)-M_t(i))\right\}} \notag\\
	&= \frac{\exp\{-\eta (x_t(i)-M_t(i))\}}{\sum_{j=1}^d f_t(j)\exp\left\{-\eta(x_t(i)-M_t(i))\right\}}
\end{align}

For any $f^*\in\F$,
\begin{align}\label{eq:head-local}
	\inner{f_t-f^*, x_t} = \inner{f_t-g_{t+1},x_t-M_{t}} + \inner{f_t-g_{t+1}, M_t} + \inner{g_{t+1}-f^*, x_t}
\end{align}
First observe that 
\begin{align}
	\label{eq:part-one-local}
	\inner{f_t-g_{t+1},x_t-M_{t}} \le \norm{f_t-g_{t+1}}_t \norm{x_t - M_t}_t^*  \ .
\end{align}
Now, since $\nabla^2\cR$ is diagonal,
\begin{align*}
	\norm{f_t-g_{t+1}}_t^2 = \sum_{i=1}^d (f_t(i)-g_{t+1}(i))^2/f_t(i) = -1 + \sum_{i=1}^d f_t(i)(g_{t+1}(i)/f_t(i))^2 
\end{align*}
using the fact that both $f_t$ and $g_{t+1}$ are probability distributions. In view of \eqref{eq:ratio-of-updates},
\begin{align*}
	\norm{f_t-g_{t+1}}_t^2 = -1 + \En \left(\frac{\exp\{-Z\}}{\En\exp\left\{-Z\right\}}\right)^2
\end{align*}
where $Z$ is defined as a random variable taking on values $\eta (x_t(i)-M_t(i))$ with probability $f_t(i)$. Then, if almost surely $\En Z- Z \leq a/2$,
$$\En \left(\frac{\exp\{-Z\}}{\En\exp\left\{-Z\right\}}\right)^2-1 \leq \En \left(\frac{\exp\{-Z\}}{\exp\left\{-\En Z\right\}}\right)^2 -1 = \En\exp\{2(\En Z - Z)\}-1 \leq 4\left(\frac{e^a-a-1}{a^2}\right) \mbox{var}(Z) $$
since the function $(e^y-y-1)/y^2$ is nondecreasing over reals. As long as $|\eta (x_t(i)-M_t(i))|\leq 1/4$, we can guarantee that $\En Z-Z < 1/2$, yielding
$$\norm{f_t-g_{t+1}}_t\leq 2\sqrt{\En Z^2} = 2\sqrt{\sum_{i=1}^d f_t(i) (\eta (x_t(i)-M_t(i)))^2}=2\eta \|x_t - M_t\|_t^*$$
Combining with \eqref{eq:part-one-local}, we have
\begin{align}
	\inner{f_t-g_{t+1},x_t-M_{t}} \le 2\eta(\norm{x_t - M_t}_t^*)^2  \ .
\end{align}

The rest similar to the proof of Lemma~\ref{lem:two-step-MD}. We have
\begin{align}
	\label{eq:part-three-local}
	\inner{f_t-g_{t+1}, M_t} \leq \frac{1}{\eta} \left(D_\cR(g_{t+1},g_t) - D_\cR(g_{t+1},f_t) - D_\cR(f_t,g_t)  \right) \ .
\end{align}
and
\begin{align}
	\label{eq:part-two-local}
	\inner{g_{t+1}-f^*, x_t} 
&\leq \frac{1}{\eta}\left(D_\cR(f^*,g_{t})  - D_\cR(f^*,g_{t+1}) - D_\cR(g_{t+1},g_{t}) \right),
\end{align}
We conclude that 
\begin{align*}
	\inner{f_t-f^*, x_t} 
	&\leq 2\eta(\norm{x_t - M_t}_t^*)^2 \\
	& +\frac{1}{\eta} \left(D_\cR(g_{t+1},g_t) - D_\cR(g_{t+1},f_t) - D_\cR(f_t,g_t) \right) \\
	& + \frac{1}{\eta}\left(D_\cR(f^*,g_{t})  - D_\cR(f^*,g_{t+1}) - D_\cR(g_{t+1},g_{t})) \right)\\
	&\le 2\eta(\norm{x_t - M_t}_t^*)^2
 + \frac{1}{\eta}\left(D_\cR(f^*,g_{t})  - D_\cR(f^*,g_{t+1})  - D_\cR(g_{t+1},f_t) \right)
\end{align*}
Summing over $t=1,\ldots,T$ yields, for any $f^*\in\F$,
$$\sum_{t=1}^T \inner{f_t-f^*, x_t} \leq 2\eta \sum_{t=1}^T (\norm{x_t - M_t}_t^*)^2 + \frac{ \log d}{\eta}$$
%Choosing $\eta = \sqrt{\frac{\log d}{2\sum_{t=1}^T (\norm{x_t - M_t}_t^*)^2}}$ balances the two terms.
\end{proof}

\begin{proof}[\textbf{Proof of Lemma~\ref{lem:bandit}}]
In view of Lemma~\ref{lem:self_conc_full_info}, for any $f^*\in\F$
\begin{align*}
	\sum_{t=1}^T \inner{h_t,\tilde{x}_t} - \sum_{t=1}^T \inner{f^*,\tilde{x}_t} &\leq \eta^{-1}\cR(f^*) + 2\eta\sum_{t=1}^T (\|\tilde{x}_t-M_t\|_t^*)^2 \\
	&= \eta^{-1}\cR(f^*) + 2\eta\sum_{t=1}^T n^2 (\inner{f_t, x_t-M_t})^2 \left(\left\|\varepsilon_t \lambda_{i_t}^{1/2} \ev{i_t}\right\|_t^*\right)^2\\
	&\le \eta^{-1}\cR(f^*) + 2\eta\sum_{t=1}^T n^2 (\inner{f_t, x_t-M_t})^2 \\
	&\leq \eta^{-1}\cR(f^*) + 2\eta n^2\sum_{t=1}^T  \|x_t-M_t\|^2 \ .
\end{align*}
where for simplicity we use the Euclidean norm and use the assumption $\|f_t\|\leq 1$; any primal-dual pair of norms will work here. It is easy to verify that $\tilde{x}_t$ is an unbiased estimate of $x_t$ and $\E f_t = h_t$. Thus, by the standard argument and the above upper bound,
\begin{align*}
	\En\left[ \sum_{t=1}^T \inner{f_t,x_t} - \sum_{t=1}^T \inner{f^*,x_t} \right]
	&= \En\left[ \sum_{t=1}^T \inner{h_t,x_t} - \sum_{t=1}^T \inner{f^*,x_t} \right]\\
	&=\En\left[ \sum_{t=1}^T \inner{h_t,\tilde{x}_t} - \sum_{t=1}^T \inner{f^*,\tilde{x}_t} \right]   \\
	&\le \eta^{-1}\cR(f^*) + 2\eta\sum_{t=1}^T n^2\E{ (\inner{f_t, x_t-M_t})^2 }\\
	&\leq \eta^{-1}\cR(f^*) + 2\eta n^2\sum_{t=1}^T  \En\left[\|x_t-M_t\|^2\right] \ .
\end{align*}
The second statement follows immediately.
\end{proof}

\begin{proof}[\textbf{Proof of Lemma~\ref{lem:two-step-MD-learn}}]
First note that by Lemma \ref{lem:two-step-MD} we have that for the $M_t$ chosen in the algorithm, 
\begin{align*}
\sum_{t=1}^T \inner{f_t, x_t} - \sum_{t=1}^T \inner{f^*, x_t} & \leq  \eta^{-1}R_{\max}^2 +  \frac{\eta}{2}  \sum_{t=1}^T \|x_t-M_t\|_*^2 \\
& \le \eta^{-1}R_{\max}^2 +  \frac{\eta}{2}  \sum_{t=1}^T \sum_{\pi \in \Pi} q_t(\pi) \|x_t-M^\pi_t\|_*^2  & \textrm{(Jensen's Inequality)}\\
& \le \eta^{-1}R_{\max}^2 +  \frac{\eta}{2}  \left(\frac{4e}{e-1}\right)\left( \inf_{\pi \in \Pi} \sum_{t=1}^T  \|x_t-M^\pi_t\|_*^2  + \log \left| \Pi\right|\right)
\end{align*}
where the last step is due to Corollary 2.3 of \cite{PLG}. Indeed, the updates for $q_t$'s are exactly the experts algorithm with pointwise loss at each round $t$ for expert $\pi \in \Pi$ given by $\norm{M^\pi_t - x_t}_*^2$. Also as each $M^\pi_t \in \X$ the unit ball of dual norm, we can conclude that $\norm{M^\pi_t - x_t}_*^2 \le 4$ which is why we have a scaling by factor $4$. Simplifying leads to the bound in the lemma.
\end{proof}

\begin{proof}[\textbf{Proof of Lemma~\ref{lem:bandit1}}]
In view of Lemma~\ref{lem:self_conc_full_info}, for any $f^*\in\F$
\begin{align*}
	\sum_{t=1}^T \inner{h_t,\tilde{x}_t} - \sum_{t=1}^T \inner{f^*,\tilde{x}_t} &\leq \eta^{-1}\cR(f^*) + 2\eta\sum_{t=1}^T (\|\tilde{x}_t-M_t\|_t^*)^2 \\
	&= \eta^{-1}\cR(f^*) + 2\eta\sum_{t=1}^T n^2 (\inner{f_t, x_t-M_t})^2 \left(\left\|\varepsilon_t \lambda_{i_t}^{1/2} \ev{i_t}\right\|_t^*\right)^2\\
	&\le \eta^{-1}\cR(f^*) + 2\eta n^2 \sum_{t=1}^T  (\inner{f_t, x_t-M_t})^2 
\end{align*}
It is easy to verify that $\tilde{x}_t$ is an unbiased estimate of $x_t$ and $\E f_t = h_t$. Thus, by the standard argument and the above upper bound we get,
\begin{align*}
		\En\left[ \sum_{t=1}^T \inner{f_t,x_t} - \sum_{t=1}^T \inner{f^*,x_t} \right]
	&= \En\left[ \sum_{t=1}^T \inner{h_t,x_t} - \sum_{t=1}^T \inner{f^*,x_t} \right]\\
	&=\En\left[ \sum_{t=1}^T \inner{h_t,\tilde{x}_t} - \sum_{t=1}^T \inner{f^*,\tilde{x}_t} \right]   \\
	&\le \eta^{-1}\cR(f^*) + 2\eta n^2 \E{\sum_{t=1}^T  (\inner{f_t, x_t-M_t})^2 }
\end{align*}
This proves the first inequality of the Lemma. Now  by Jensen's inequality, the above bound can be simplified as:
\begin{align*}
		\En\left[ \sum_{t=1}^T \inner{f_t,x_t} - \sum_{t=1}^T \inner{f^*,x_t} \right]
	&\le \eta^{-1}\cR(f^*) + 2\eta n^2 \E{\sum_{t=1}^T  (\inner{f_t, x_t-M_t})^2 }\\
	&\le \eta^{-1}\cR(f^*) + 2\eta n^2 \E{\sum_{t=1}^T  \sum_{\pi \in \Pi} q_t(\pi) (\inner{f_t, x_t-M^\pi_t})^2 }\\
	& \le \eta^{-1}\cR(f^*) +   8\eta n^2 \left(\frac{e}{e-1} \right)\left( \En\inf_{\pi \in \Pi}\sum_{t=1}^T (\inner{f_t, x_t-M^\pi_t})^2  + \log \left|\Pi\right| \right)~.
\end{align*}
where the last step is due to Corollary 2.3 of \cite{PLG}. Indeed, the updates for $q_t$'s are exactly the experts algorithm with point-wise loss at each round $t$ for expert $\pi \in \Pi$ given by $ (\inner{f_t, x_t-M^\pi_t})^2$. Also as each $M^\pi_t \in \X$ the unit ball of dual norm, hence we can conclude that $ (\inner{f_t, x_t-M^\pi_t})^2 \le 4$ which is why we have a scaling by factor $4$. Further since $\|f_t\|\leq 1$ we can conclude that : 
\begin{align*}
		\En\left[ \sum_{t=1}^T \inner{f_t,x_t} - \sum_{t=1}^T \inner{f^*,x_t} \right]
	&\le \eta^{-1}\cR(f^*) +   8\eta n^2 \left(\frac{e}{e-1} \right)\left( \En\inf_{\pi \in \Pi}\sum_{t=1}^T \norm{ x_t-M^\pi_t}^2  + \log \left|\Pi\right| \right)\\
	&\le \eta^{-1}\cR(f^*) +   13\eta n^2 \left( \En\inf_{\pi \in \Pi}\sum_{t=1}^T \norm{ x_t-M^\pi_t}^2  + \log \left|\Pi\right| \right)~.
\end{align*}
This concludes the proof.
\end{proof}

\begin{proof}[\textbf{Proof of Lemma~\ref{lem:two-step-MD-learn-partial}}]
First note that by Lemma \ref{lem:two-step-MD}, since $M_t^{\pi_t}$ is the predictable process we use, we have deterministically that,
\begin{align*}
\sum_{t=1}^T \inner{f_t, x_t} - \sum_{t=1}^T \inner{f^*, x_t} & \leq  \eta^{-1}R_{\max}^2 +  \frac{\eta}{2}  \sum_{t=1}^T \|x_t-M^{\pi_t}_t\|_*^2 
\end{align*}
Hence we can conclude that expected regret is bounded as :
\begin{align}\label{eq:expregban1}
\E{\sum_{t=1}^T \inner{f_t, x_t} - \sum_{t=1}^T \inner{f^*, x_t}} & \leq  \eta^{-1}R_{\max}^2 +  \frac{\eta}{2}  \E{\sum_{t=1}^T \|x_t-M^{\pi_t}_t\|_*^2 }
\end{align}
This proves the first inequality in the lemma.
However note that the update for $q_t$'s is using SCRiBLe for multiarmed bandit algorithm where the pointwise loss for any $\pi \in \Pi$ at round $t$ given by $\norm{x_t - M_t^\pi}_*^2$. Also note that maximal value of loss is bounded  by $\max_{M_t , x_t } \norm{x_t - M_t^\pi}_* \le 4$. Hence, using Lemma \ref{lem:non-stoch-multiarmed} with $s = 4$ and step size $1/32 |\Pi|^2$, we conclude that
$$
\E{\sum_{t=1}^T \|x_t-M^{\pi_t}_t\|_*^2 } \le 2 \inf_{\pi \in \Pi} \sum_{t=1}^T \|x_t-M^{\pi}_t\|_*^2 + 64 |\Pi|^3 \log(T |\Pi|)
$$
Using this in Equation~\eqref{eq:expregban1} we obtain
$$
\E{\sum_{t=1}^T \inner{f_t, x_t} - \sum_{t=1}^T \inner{f^*, x_t}}  \leq  \eta^{-1}R_{\max}^2 +  \eta \left(\inf_{\pi \in \Pi} \sum_{t=1}^T \|x_t-M^{\pi}_t\|_*^2 + 32  |\Pi|^3 \log(T |\Pi|) \right)
$$
\end{proof}

\begin{proof}[\textbf{Proof of Lemma~\ref{lem:bandit3}}]
In view of Lemma~\ref{lem:self_conc_full_info}, for any $f^*\in\F$
\begin{align*}
	\sum_{t=1}^T \inner{h_t,\tilde{x}_t} - \sum_{t=1}^T \inner{f^*,\tilde{x}_t} &\leq \eta^{-1}\cR(f^*) + 2\eta\sum_{t=1}^T (\|\tilde{x}_t-M^{\pi_t}_t\|_t^*)^2 \\
	&= \eta^{-1}\cR(f^*) + 2\eta\sum_{t=1}^T n^2 (\inner{f_t, x_t-M^{\pi_t}_t})^2 \left(\left\|\varepsilon_t \lambda_{i_t}^{1/2} \ev{i_t}\right\|_t^*\right)^2\\
	&\le \eta^{-1}\cR(f^*) + 2\eta n^2 \sum_{t=1}^T  (\inner{f_t, x_t-M^{\pi_t}_t})^2\ .
\end{align*}
We can bound expected regret of the algorithm as: 
\begin{align}
	\En_{\pi_{1:T}, i_{1:T}}\left[ \sum_{t=1}^T \inner{f_t,x_t} - \sum_{t=1}^T \inner{f^*,x_t} \right]
		&=  \sum_{t=1}^T \En_{i_{1:t-1},\pi_{1:t}}\left[\inner{h_t,x_t}\right] - \sum_{t=1}^T \inner{f^*,x_t} \notag\\
		&=  \sum_{t=1}^T \En_{i_{1:t},\pi_{1:t}}\left[\inner{h_t,\tilde{x}_t}\right] - \sum_{t=1}^T \Es{i_t}{\inner{f^*,\tilde{x}_t}} \notag\\
	&=\En\left[ \sum_{t=1}^T \inner{h_t,\tilde{x}_t} - \sum_{t=1}^T \inner{f^*,\tilde{x}_t} \right]   \notag\\
	&\leq \eta^{-1}\cR(f^*) + 2\eta n^2 \E{\sum_{t=1}^T  (\inner{f_t, x_t-M^{\pi_t}_t})^2}\label{eq:expregban2}
\end{align}
This gives the first inequality of the Lemma.
However note that the update for $q_t$'s the distribution over set $\Pi$ is obtained by running the SCRiBLe for multi-armed bandit algorithm where pointwise loss for any $\pi \in \Pi$ at round $t$ given by $(\ip{f_t}{x_t - M_t^{\pi}})^2$. Also note that maximal value of loss is bounded  by $4$. Hence using Lemma \ref{lem:non-stoch-multiarmed} with  $s = 4$ and step size $1/32 |\Pi|^2$ we conclude by the regret bound in that lemma that 
$$
\E{\sum_{t=1}^T  (\inner{f_t, x_t-M^{\pi_t}_t})^2} \le 2 \E{\inf_{\pi \in \Pi}\sum_{t=1}^T  (\inner{f_t, x_t-M^{\pi}_t})^2 + 64 |\Pi|^3 \log(T |\Pi|)}
$$
Plugging this back in Equation~\eqref{eq:expregban2} we conclude that
\begin{align*}
\E{\Reg_T} &\leq \eta^{-1}\cR(f^*) + 4 \eta n^2 \left( \E{ \inf_{\pi \in \Pi} \sum_{t=1}^T  (\inner{f_t, x_t-M^{\pi}_t})^2} + 32 |\Pi|^3 \log(T |\Pi|) \right)\\
& \leq \eta^{-1}\cR(f^*) + 4 \eta n^2 \left(  \E{ \inf_{\pi \in \Pi} \sum_{t=1}^T  \norm{x_t-M^{\pi}_t}^2} + 32 |\Pi|^3 \log(T |\Pi|)\right) \ .
\end{align*}

\end{proof}

\begin{proof}[\textbf{Proof of Lemma \ref{lem:fpl}}]
To show admissibility using the particular randomized strategy $q_t$ given in the lemma, we need to show that 
\begin{align*}
\sup_{x_t\in C_t(x_{1:t-1})} \left\{ \En_{f \sim q_t} f^\tr x_t + \Relax{T}{\F}{x_1,\ldots,x_t}\right\} \le \Relax{T}{\F}{x_1,\ldots,x_{t-1}}
\end{align*}
The distribution $q_t$ is defined by first drawing $z_{t+1}\sim D_{t+1},\ldots,z_T \sim D_T$ and $\epsilon_{t+1},\ldots \epsilon_T$ Rademacher random variables, and then calculating $f_t=f_t(z_{t+1:T},\epsilon_{t+1:T})$ as in \eqref{eq:def_general_fpl}. Hence,
\begin{align*}
\sup_{x_t\in C_t(x_{1:t-1})} \left\{ \En_{f \sim q_t} f^\tr x_t + \Relax{T}{\F}{x_1,\ldots,x_t}\right\} 
& = \sup_{x_t\in C_t(x_{1:t-1})} \left\{ \Eunder{z_{t+1:T}}{\epsilon_{t+1:T}} f_t^\tr x_t + \Eunder{z_{t+1:T}}{\epsilon_{t+1:T}} \left\| C \sum_{i=t+1}^T \epsilon_i z_i - \sum_{i=1}^t x_i \right\| \right\}\\
& \le \Eunder{z_{t+1:T}}{\epsilon_{t+1:T}} \sup_{x_t\in C_t(x_{1:t-1})} \left\{ f_t^\tr x_t +  \left\| C \sum_{i=t+1}^T \epsilon_i z_i - \sum_{i=1}^t x_i \right\| \right\}
\end{align*}
Now, with $f_t$ defined as 
\begin{align*}
	f_t = \argmin{g \in \F} \sup_{x_t\in C_t(x_{1:t-1})} \left\{\inner{g,x_t} + \norm{ C \sum_{i=t+1}^T \epsilon_i z_i  - \sum_{i=1}^{t} x_i } \right\}
\end{align*}
for any given $z_{t+1:T},\epsilon_{t+1:T}$, we have
\begin{align*}
\sup_{x_t\in C_t(x_{1:t-1})} & \left\{ f_t^\tr x_t + \left\| C \sum_{i=t+1}^T \epsilon_i z_i - \sum_{i=1}^t x_i \right\| \right\}  = \inf_{g \in \F} \sup_{x_t\in C_t(x_{1:t-1})} \left\{ g^\tr x_t +  \left\| C \sum_{i=t+1}^T \epsilon_i z_i - \sum_{i=1}^t x_i \right\| \right\}
\end{align*}
We can conclude that for this choice of $q_t$, 
\begin{align*}
\sup_{x_t\in C_t(x_{1:t-1})} &\left\{ \Es{f \sim q_t}{f^\tr x_t} + \Relax{T}{\F}{x_1,\ldots,x_t}\right\}  \le \Eunder{z_{t+1:T}}{\epsilon_{t+1:T}} \inf_{g \in \F} \sup_{x_t\in C_t(x_{1:t-1})} \left\{g^\tr x_t +  \left\| C \sum_{i=t+1}^T \epsilon_i z_i - \sum_{i=1}^t x_i \right\|\right\}\\
& =  \Eunder{z_{t+1:T}}{\epsilon_{t+1:T}} \inf_{g \in \F} \sup_{p \in \Delta(C_t(x_{1:t-1}))} \Es{x_t \sim p}{ g^\tr x_t +   \left\| C \sum_{i=t+1}^T \epsilon_i z_i - \sum_{i=1}^t x_i \right\| }\\
& = \Eunder{z_{t+1:T}}{\epsilon_{t+1:T}}\sup_{p \in \Delta(C_t(x_{1:t-1}))} \inf_{g \in \F}  \left\{\Es{x_t \sim p}{ g^\tr x_t } +  \En_{x_t \sim p}  \left\| C \sum_{i=t+1}^T \epsilon_i z_i - \sum_{i=1}^t x_i \right\|  \right\} \\
&= \Eunder{z_{t+1:T}}{\epsilon_{t+1:T}}\sup_{p \in \Delta(C_t(x_{1:t-1}))} \left\{ -\left\|  \Es{x_t \sim p}{ x_t } \right\| +  \En_{x_t \sim p}  \left\| C \sum_{i=t+1}^T \epsilon_i z_i - \sum_{i=1}^t x_i \right\|  \right\}
\end{align*}
In the next to last step we appealed to the minimax theorem which by linearity of the expression in $g$ and the fact that  $\F$ is a compact convex set; furthermore, the term in the expectation is linear in $p$. 
By triangle inequality, 
\begin{align*}
-\left\|  \Es{x_t \sim p}{ x_t } \right\| +  \En_{x_t \sim p}  \left\| C \sum_{i=t+1}^T \epsilon_i z_i - \sum_{i=1}^t x_i \right\| 
&\leq\En_{x_t \sim p}  \left\| C \sum_{i=t+1}^T \epsilon_i z_i - \sum_{i=1}^{t-1} x_i + \Es{x_t \sim p}{ x_t} - x_t \right\| \\
&\le \En_{x_t,x'_t \sim p}  \left\| C \sum_{i=t+1}^T \epsilon_i z_i - \sum_{i=1}^{t-1} x_i + x'_t - x_t \right\| \\
& = \En_{x_t,x'_t \sim p} \En_{\epsilon_t} \left\| C \sum_{i=t+1}^T \epsilon_i z_i - \sum_{i=1}^{t-1} x_i + \epsilon_t(x'_t - x_t) \right\| 
\end{align*}
where we introduced a Rademacher random variable $\epsilon_t$ via the standard symmetrization argument. We now introduce ``centering'' by $M_t(x_{1:t-1})$. The above expression is equal to
\begin{align*}
&\En_{x_t,x'_t \sim p} \En_{\epsilon_t} \left\| C \sum_{i=t+1}^T \epsilon_i z_i - \sum_{i=1}^{t-1} x_i + \epsilon_t(x'_t-M_t(x_{1:t-1})) +\epsilon_t(M_t(x_{1:t-1})-x_t) \right\| \\
&\leq \En_{x_t\sim p} \En_{\epsilon_t} \left\| C \sum_{i=t+1}^T \epsilon_i z_i - \sum_{i=1}^{t-1} x_i + 2\epsilon_t(x_t-M_t(x_{1:t-1})) \right\| 
\end{align*}
Hence,
\begin{align*}
&\Eunder{z_{t+1:T}}{\epsilon_{t+1:T}}\sup_{p \in \Delta(C_t(x_{1:t-1}))} \left\{ -\left\|  \Es{x_t \sim p}{ x_t } \right\| +  \En_{x_t \sim p}  \left\| C \sum_{i=t+1}^T \epsilon_i z_i - \sum_{i=1}^t x_i \right\|  \right\}\\
&= \Eunder{z_{t+1:T}}{\epsilon_{t+1:T}} \sup_{p \in \Delta(C)}\En_{z_t\sim p} \En_{\epsilon_t} \left\| C \sum_{i=t+1}^T \epsilon_i z_i - \sum_{i=1}^{t-1} x_i + 2\epsilon_t z_t \right\| 
\end{align*}
where in the last step we pass to the set of distributions on $C=\{z: \|z\|\leq \sigma_t\}$. By Assumption~\ref{asm:fpl-linear}, the last expression is upper bounded by
\begin{align*}
\Eunder{z_{t+1:T}}{\epsilon_{t+1:T}} \En_{z_t \sim D_t}\En_{\epsilon_t} \left\| C \sum_{i=t+1}^T \epsilon_i z_i - \sum_{i=1}^{t-1} x_i +   C \epsilon_t z_t \right\|
& = \Relax{T}{\F}{x_1,\ldots,x_{t-1}}
\end{align*}

\end{proof}

\begin{lemma}\label{lem:l1linffpl}
Consider the case when $\X$ is the $\ell_\infty^N$ unit ball and $\F$ is the $\ell_1^N$ unit ball. Let $R_t$ be any random vector and define $j^*_t = \argmax{j \in [d]} |R_t[j]|$. Let
$$
f_t(R_t) = \argmin{f : \norm{f}_1 \le 1} \left\{\sigma_t \sum_{i\ne j^*_t} |f[i]| +  \sigma_t f[j^*_t] \sign(R_t[j^*_t]) + \ip{f}{M_t} \right\},
$$ 
where $M_t$ is any fixed vector in $\reals^N$. Then
\begin{align*}
\Es{R_t}{ \sup_{z : \norm{z}_\infty \le \sigma_t} \left\{\ip{f_t(R_t)}{z + M_t} + \norm{R_t + z}_\infty \right\}} & \le  \Es{R_t}{\inf_{f \in \F} \sup_{z : \norm{z}_\infty \le \sigma_t} \left\{\ip{f}{z + M_t} + \norm{R_t + z}_\infty \right\}} + 4\ \P\left( {\mathcal E}_t^c \right)
\end{align*} 
where ${\mathcal E}_t$ is the event that the largest two coordinates of $R_t$ are separated by at least $4 \sigma_t$.
\end{lemma}
\begin{proof}[\textbf{Proof of Lemma~\ref{lem:l1linffpl}}]
	For any given vector $R_t,M_t \in \reals^N$ and any $f\in\F$, 
	\begin{align*}
	\sup_{z : \norm{z}_\infty \le \sigma_t} \left\{\ip{f}{z + M_t} + \norm{R_t + z}_\infty \right\} 
	%& = \sup_{x : \norm{x}_\infty \le \sigma_t} \left\{\ip{f}{x } + \norm{R + x}_\infty \right\} + \ip{f}{M}\\
	& = \sup_{z \in \{-1,1\}^d} \left\{\sigma_t \ip{f}{z } + \norm{R_t + \sigma_t z}_\infty \right\} + \ip{f}{M_t}
	\end{align*}
	Leaving out the $\ip{f}{M_t}$ term, we can further rewrite the above supremum as
	\begin{align*}
		\sup_{z \in \{-1,1\}^d} \left\{\sigma_t \sum_{i=1}^d f[i] \cdot z[i] + \max_{j \in [d]} |R_t[j] + \sigma_t z[j]| \right\}  
		&= \max_{j \in [d]} \sup_{z \in \{-1,1\}^d} \left\{\sigma_t \sum_{i=1}^d f[i] \cdot z[i] +  \left|R_t[j] + \sigma_t z[j]\right| \right\} 
	\end{align*}
	By optimizing over coordinates $i\neq j$, this is equal to
	\begin{align*}
	&\max_{j \in [d]}\left\{ \sigma_t \sum_{i \ne j} |f[i]| +  \max\{ |R_t[j] + \sigma_t| + \sigma_t f[j] ~,~ |R_t[j] - \sigma_t| - \sigma_t f[j] \} \right\} \\
	& = \sigma_t \norm{f}_1 + \max_{j \in [d]}\left\{  - \sigma_t |f[j]| +  \max\{ |R_t[j] + \sigma_t| + \sigma_t f[j] ~,~ |R_t[j] - \sigma_t| - \sigma_t f[j] \} \right\} 
	\end{align*}
	Under the event $\mathcal E_t$, the maximum over $j$ will be achieved at $j^*_t$, thus yielding
\begin{align*}
&\sigma_t \norm{f}_1 + |R_t[j^*_t]| + \sigma_t + \sigma_t |f[j^*_t]| \left(\sign(f[j^*_t]) \sign(R_t[j^*_t])  - 1\right)  \\
& = \sigma_t \norm{f}_1 + |R_t[j^*_t]| + \sigma_t - 2 \sigma_t |f[j^*_t]| \ind{\sign(f[j^*_t])  \ne \sign(R_t[j^*_t]) }  
\end{align*}
while outside of $\mathcal E_t$ the above solution can be off by at most $4$. We may also write the above expression as
\begin{align*}
|R_t[j^*_t]| + \sigma_t +  \sigma_t \sum_{i\ne j^*_t} |f[i]| +  \sigma_t f[j^*_t] \sign(R_t[j^*_t]) \ .
\end{align*}
So, under the event $\mathcal E_t$, the minimum is attained at  
$$
f_t(R_t) = \argmin{f : \norm{f}_1 \le 1} \left\{\sigma_t \sum_{i\ne j^*_t} |f[i]| +  \sigma_t f[j^*_t] \sign(R_t[j^*_t]) + \ip{f}{M_t} \right\} 
$$
and so 
$$
\sup_{z : \norm{z}_\infty \le \sigma_t} \left\{\ip{f_t(R_t)}{z + M_t} + \norm{R_t + z}_\infty \right\}  \le  \inf_{f \in \F} \sup_{z : \norm{z}_\infty \le \sigma_t} \left\{\ip{f}{z + M_t} + \norm{R_t + z}_\infty \right\} ~.
$$
On the other hand on the event $\mc{E}_t^c$,
$$
\sup_{z : \norm{z}_\infty \le \sigma_t} \left\{\ip{f_t(R_t)}{z + M_t} + \norm{R_t + z}_\infty \right\}  -  \inf_{f \in \F} \sup_{z : \norm{z}_\infty \le \sigma_t} \left\{\ip{f}{z + M_t} + \norm{R_t + z}_\infty \right\} \le 4
$$
and so
$$
\sup_{z : \norm{z}_\infty \le \sigma_t} \left\{\ip{f_t(R_t)}{z + M_t} + \norm{R_t + z}_\infty \right\}  \le  \inf_{f \in \F} \sup_{z : \norm{z}_\infty \le \sigma_t} \left\{\ip{f}{z + M_t} + \norm{R_t + z}_\infty \right\}   + 4\ind{\mc{E}_t^c}~.
$$
Taking expectation proves the result.
\end{proof}

\begin{proof}[\textbf{Proof of Theorem~\ref{thm:fpl_update}}]
From Lemma \ref{lem:fpl} we have that the randomized strategy which at time $t$, draws $z_{t+1},\ldots,z_T$ from $D_{t+1},\ldots,D_T$ respectively and Rademacher random variables $\epsilon=(\epsilon_{t+1},\ldots,\epsilon_T)$, and then picks
$$
	f_t = \argmin{g \in \F} \sup_{x_t\in C_t(x_{1:t-1})} \left\{\inner{g,x_t} + \norm{ C \sum_{i=t+1}^T \epsilon_i z_i  - \sum_{i=1}^{t-1} x_i - x_t }_* \right\}
$$
is admissible w.r.t. relaxation 
$$
\Relax{T}{\F}{x_1,\ldots,x_t} = \Eunderone{z_{t+1}\sim D_{t+1},\ldots z_T \sim D_{T}}\En_{\epsilon}   \left\|C \sum_{i=t+1}^T \epsilon_i z_i - \sum_{i=1}^t x_i \right\|_*~.
$$
However by Lemma \ref{lem:l1linffpl}, we have that for the randomized algorithm that at time $t$, draws $z_{t+1},\ldots,z_T$ from $D_{t+1},\ldots,D_T$ respectively and Rademacher random variables $\epsilon=(\epsilon_{t+1},\ldots,\epsilon_T)$, and then picks 
\begin{align}
f_t(R_t) = \argmin{f : \norm{f}_1 \le 1} \left\{\sigma_t \sum_{i\ne j^*_t} |f[i]| +  \sigma_t f[j^*_t] \sign(R_t[j^*_t]) + \ip{f}{M_t} \right\}~,\label{eq:fpllong}
\end{align}
we have that
$$
\Es{R_t}{ \sup_{z : \norm{z}_\infty \le \sigma_t} \left\{\ip{f_t(R_t)}{z + M_t} + \norm{R_t + z}_\infty \right\}}  \le  \Es{R_t}{\inf_{f \in \F} \sup_{z : \norm{z}_\infty \le \sigma_t} \left\{\ip{f}{z + M_t} + \norm{R_t + z}_\infty \right\}} + 4\ \P\left( {\mathcal E}_t^c \right)
$$
Hence we can conclude that the Randomized strategy that at time $t$, draws $z_{t+1},\ldots,z_T$ from $D_{t+1},\ldots,D_T$ respectively and Rademacher random variables $\epsilon=(\epsilon_{t+1},\ldots,\epsilon_T)$, and then picks $f_t(R_t) = \argmin{f : \norm{f}_1 \le 1} \left\{\sigma_t \sum_{i\ne j^*_t} |f[i]| +  \sigma_t f[j^*_t] \sign(R_t[j^*_t]) + \ip{f}{M_t} \right\}$ is admissible w.r.t. the relaxation,
\begin{align}
\Relax{T}{\F}{x_1,\ldots,x_t} = \Eunderone{z_{t+1}\sim D_{t+1},\ldots z_T \sim D_{T}}\En_{\epsilon}   \left\|C \sum_{i=t+1}^T \epsilon_i z_i - \sum_{i=1}^t x_i \right\|_* +\, 4\, \sum_{i=t+1}^T \P\left( {\mathcal E}_t^c \right)~. \label{eq:newrel}
\end{align}
Hence as mentioned in Equation \eqref{eq:sum_cond_exp_bdd_by_relax} we can conclude that the expected regret of the randomized strategy that plays $f_t(R_t)$ on round $t$ is bounded as
$$
\E{\Reg_T} \le C\ \En_{z_{1:T}} \En_{\epsilon} \left\|\sum_{t=1}^T \epsilon_t z_t\right\|_* + 4\ \sum_{t=1}^T  \P\left( {\mathcal E}_t^c \right) \ .
$$

Now we claim that the update in Equation~\eqref{eq:fpllong} is same as the one in Equation~\eqref{eq:fplup} given in the theorem statement and so the above regret bound is true for the update provided in the theorem.  To prove this, we first show that the $f_t(R_t)$ given in Equation~\eqref{eq:fpllong} is on a vertex of the $\ell_1$ ball. To see this note that we can rewrite the minimization as 
$$
\argmin{s : \{\pm 1\}^d} \argmin{g : \forall i \in [d], g[i] \ge 0, \sum_{i=1}^d g[i] \le 1 } \left\{\sigma_t \sum_{i=1}^d  g[i] +  \sigma_t s[j^*_t] g[j^*_t] \sign(R_t[j^*_t])  + \sum_{i=1}^d s[i] g[i] M_t[i] \right\}
$$
and $f_t(R_t) = (s[1] g[1], \ldots,s[d] g[d])$. That is vector $s$ is the sign vector, $\sign(f_t(R_t))$, and vector $g$ is the magnitude vector, $|f_t|$. Further note that given $s \in \{\pm 1\}^d$, the minimization problem in terms of $g$ is linear in $g$. Hence the solution 
will be at a vertex of the set $\{g : \forall i \in [d] , g[i] \ge 0, \sum_{i=1}^d g[i] \le 1\}$ as its a linear optimization problem. Hence either $g = 0$ or $g = e_i$ for some $i \in [d]$. However the solution is clearly not $f_t(R_t) = 0$ as the minimum has to at least be negative unless $M_t$ and $R_t$ are both 0. Thus we see that $g = e_i$ for some $i$ and so $f_t(R_t)$ is of form $s[i] e_i$ and so $g$ is on the vertex of the $\ell_1^N$ ball. Hence we conclude that 
update in Equation~\eqref{eq:def_general_fpl_with_R} can be rewritten as $f_t(R_t) = s_t e_{i_t}$ where
\begin{align*}
(i_t,s_t) & = \argmin{ i \in [d], s \in \{\pm 1\}} \left\{\sigma_t  \ind{i \ne j^*_t} +  \sigma_t s \ind{i_t = j^*_t} \sign(R_t[j^*_t]) + s M_t[i] \right\}
\end{align*}
Let $i^*_t = \argmax{i \in [d], i \ne j^*_t} |M_t[i]|$ it is easy to see that the $f_t(R_t) = s_t e_{i_t}$ is given as follows : 
\begin{align*}
f_t(R_t) = \left\{ \begin{array}{cl}
- \sign(M_t[i^*_t]) e_{i^*_t} & \textrm{if }\sigma_t - |M_t[i^*_t]| <  - \left|\sigma_t\ \sign(R_t[j^*_t]) + M_t[j^*_t]\right| \\
- \sign(\sigma_t R_t[j^*_t] + M_t[j^*]) e_{j^*_t}  & \textrm{otherwise}
\end{array}
\right.
\end{align*}
Hence we have shown that the update in Equation~\eqref{eq:fplup} is admissible w.r.t. relaxation in Equation~\eqref{eq:newrel} and so enjoys the expected regret bound :
$$
\E{\Reg_T} \le C\ \En_{z_{1:T}} \En_{\epsilon} \left\|\sum_{t=1}^T \epsilon_t z_t\right\|_* + 4\ \sum_{t=1}^T  \P\left( {\mathcal E}_t^c \right) \ ,
$$
thus proving the theorem.
\end{proof}

\begin{proof}[\textbf{Proof of Corollary~\ref{cor:fpl_simplex_update}}]
For the case when $\F$ is the simplex, since for each $f \in \F$ and each $i \in [d]$, $f[i] \ge 0$, if we add an arbitrary number $B$ to each coordinate of $x_t \in [-1,1]^d$, the regret remains unchanged, that is, 
\begin{align*}
 \sum_{t=1}^T  \ip{f_t}{x_t} - \inf_{f \in \F} \sum_{t=1}^T \ip{f}{x_t} =  \sum_{t=1}^T  \ip{f_t}{x_t + B\, \mathbf{1} } - \inf_{f \in \F} \sum_{t=1}^T \ip{f}{x_t + B\, \mathbf{1}}
\end{align*}
where $\mathbf{1} = (1,\ldots,1)\in \reals^d$. Hence, let us consider adding to each coordinate of every $x_t$ a large  constant $B<0$ (for instance think of $B < - e^{T^2}$ or smaller), and set $\tilde{x}_t = x_t + B \mathbf{1}$ and $\tilde{M}_t = M_t + B \mathbf{1}$. Notice that with predictable process given by $\tilde{M}_t$ and with adversary playing $\tilde{x}_t$ we still have that $\norm{\tilde{x}_t - \tilde{M_t}} = \norm{z_t} \le \sigma_t$. We now claim that the algorithm for the $\ell_1$ ball from the previous section operating on $\tilde{x}_t$'s has the following properties: it (a) produces solutions within simplex, (b) does not require the knowledge of $B$, and (c) attains a regret bound that does not depend on $B$. We will further also show that this solution is the one given in Equation~\eqref{eq:fplsimplexup} of the Corollary statement.

Let us first begin by noting that when we look at the linear game on input sequence $\tilde{x}_1,\ldots, \tilde{x}_T$, even when we take $\F$ to be all of the $\ell_1$ ball, the comparator will in fact be in the positive orthant. To see this note that since $x_t \in [-1,1]$, each $\tilde{x}_t$ is in the negative orthant. Hence, 
$$
- \inf_{i \in [d]} \sum_{t=1}^T \ip{e_i}{\tilde{x}_t}  = - \inf_{f : \norm{f}_1 \le 1} \sum_{t=1}^T \ip{f}{\tilde{x}_t} = \norm{\sum_{t=1}^T \tilde{x}_t}_{\infty}
$$ 
If we further show that each $f_t$ picked by algorithm in Theorem~\ref{thm:fpl_update} is also in the simplex then we effectively show that the algorithm from previous section can be adapted to play on the simplex by simply adding this large negative number to each coordinate of $x_t$'s. Further the randomized algorithm also enjoys the same regret bound provided in previous section and since the regret bound only depended on magnitude of $z_t = x_t - M_t = \tilde{x}_t - \tilde{M}_t$, we can conclude that the regret bound only depends on $\sigma_t$ and is independent of $B$.

Notice that $\argmax{i \in [d]} |\tilde{M_t}[i]| = \argmax{i \in [d]} -\tilde{M_t}[i] = \argmax{i \in [d]} - M_t[i] - B = \argmin{i \in [d]} M_t[i] = i^*_t$. Similarly we also have that $\argmax{i \in [d]} |\tilde{R_t}[i]| = j^*_t$ where $\tilde{R}_t = \sum_{i=1}^{t-1} \tilde{x}_i- C \sum_{i=t+1}^T \epsilon_i z_i + \tilde{M}_t$.
Now note that the algorithm of the previous section for the game where adversary plays $\tilde{x}_t$ is given by 
\begin{align*}
	f_t =  \left\{ \begin{array}{cl}
- \sign(\tilde{M}_t[i^*_t]) e_{i^*_t} & \textrm{if }\sigma_t - |\tilde{M}_t[i^*_t]| <  - \left|\sigma_t\ \sign(\tilde{R}_t[j^*_t]) + \tilde{M}_t[j^*_t]\right| \\
- \sign(\sigma_t \tilde{R}_t[j^*_t] + \tilde{M}_t[j^*]) e_{j^*_t}  & \textrm{otherwise}
\end{array}
\right.
\end{align*}
Since $B$ is a very large negative constant, we have that $\sign(\tilde{R}_t[j^*_t])  = \sign(\tilde{M}_t[j^*_t]) = -1$ and that $|\tilde{M}_t[i^*_t]| = -\tilde{M}_t[i^*_t] = - M_t[i^*_t]  - B$ and similarly, $\left|\sigma_t\ \sign(\tilde{R}_t[j^*_t]) + \tilde{M}_t[j^*_t]\right| = \sigma_t  - \tilde{M}[j^*_t]  = \sigma_t  - M_t[j^*_t] - B$. Therefore, we can rewrite $f_t$'s as
\begin{align*}
	f_t =  \left\{ \begin{array}{cl}
 e_{i^*_t} & \textrm{if } 2 \sigma_t < M_t[j^*_t] - M_t[i^*_t] \\
e_{j^*_t}  & \textrm{otherwise}
\end{array}
\right.
\end{align*}
We conclude that the randomized algorithm for the $\ell_1/\ell_\infty$ case from the previous section on the sequence given by $\tilde{x}_t$ produces $f_t$'s in the simplex. Further regret of the algorithm for on sequence $x_1,\ldots,x_T$ is same as its regret on $\tilde{x}_1,\ldots,\tilde{x}_T$ and this regret is bounded as
$$
\E{\Reg_T} \le C\ \En_{z_{1:T}} \En_{\epsilon} \left\|\sum_{t=1}^T \epsilon_t z_t\right\|_* + 4\ \sum_{t=1}^T  \P\left( {\mathcal E}_t^c \right) \ .
$$
This concludes the proof of the corollary. Notice that throughout we assumed $B$ is a negative constant with large enough magnitude so that for any $t$, $\sign(\tilde{R}_t) = -1$ (or at least this is true with very high probability). However since the result did not depend on $B$ nor does the final algorithm we can simply take $B$ to have magnitude tending to $\infty$ so that $\sign(\tilde{R}_t) = -1$ almost surely. 

\end{proof}

\begin{proof}[\textbf{Proof of Lemma~\ref{lem:banditlstar}}]
In view of Lemma~\ref{lem:self_conc_full_info}, for any $f^*\in\F$
\begin{align*}
	\sum_{t=1}^T \inner{h_t,\tilde{x}_t} - \sum_{t=1}^T \inner{f^*,\tilde{x}_t} &\leq \eta^{-1}\cR(f^*) + 2\eta\sum_{t=1}^T (\|\tilde{x}_t\|_t^*)^2 \\
	&= \eta^{-1}\cR(f^*) + 2\eta\sum_{t=1}^T n^2 (\inner{f_t, x_t})^2 \left(\left\|\varepsilon_t \lambda_{i_t}^{1/2} \ev{i_t}\right\|_t^*\right)^2\\
	&\le \eta^{-1}\cR(f^*) + 2 s\ \eta n^2 \sum_{t=1}^T  \inner{f_t, x_t} \left(\left\|\varepsilon_t \lambda_{i_t}^{1/2} \ev{i_t}\right\|_t^*\right)^2\\
	&\le \eta^{-1}\cR(f^*) + 2 s\ \eta n^2 \sum_{t=1}^T  \inner{f_t, x_t}\ .
\end{align*}
It is easy to verify that $\tilde{x}_t$ is an unbiased estimate of $x_t$ and $\E f_t = h_t$. Thus, 
\begin{align*}
	\En\left[ \sum_{t=1}^T \inner{f_t,x_t} - \sum_{t=1}^T \inner{f^*,x_t} \right]
	&= \En\left[ \sum_{t=1}^T \inner{h_t,x_t} - \sum_{t=1}^T \inner{f^*,x_t} \right]\\
	&=\En\left[ \sum_{t=1}^T \inner{h_t,\tilde{x}_t} - \sum_{t=1}^T \inner{f^*,\tilde{x}_t} \right]   \\
	&\leq \eta^{-1}\cR(f^*) +  2 s\ \eta n^2 \E{\sum_{t=1}^T  \inner{f_t, x_t}} \ .
\end{align*}
Hence we can conclude that
\begin{align*}
\E{\sum_{t=1}^T  \inner{f_t, x_t}} \le \frac{1}{1 - (2 s n^2) \eta } \left(\inf_{f \in \F}\sum_{t=1}^T \inner{f,x_t} +  \eta^{-1}\cR(f^*)  \right)
\end{align*}

\end{proof}

\begin{proof}[\textbf{Proof of Lemma \ref{lem:non-stoch-multiarmed}}]
We are interested in solving the multi-armed bandit problem using the self-concordant barrier method so we can get a regret bound in terms of the loss of the optimal arm. We do this in two steps, first we provide an algorithm for linear bandit problem over the simplex. That is we provide an algorithm for the case when learner plays on each round $q_t \in \Delta([d])$, adversary plays loss vector $x_t \in [0,s]^d$ and learner observes $\ip{q_t}{x_t}$ at the end of the round. Next we show that this bandit algorithm over the simplex can be converted into a multi-armed bandit algorithm. To this end let us first develop a linear bandit algorithm over the simplex based on self-concordant barrier algorithm (SCRiBLe).

\paragraph{Bandit algorithm over simplex:}
%%%%%%%%%%%%%%%%%%%%%%%%%%%%%%%
Note that one can rewrite the loss of any $q \in \Delta([d])$ over any $x \in [0,s]^d$ as
\begin{align*}
\ip{q}{x} & = \ip{q[1:d-1]}{x[1:d-1]} + (1 - \ip{q[1:d-1]}{\mbf{1}}) x[d] \\
&= \ip{q[1:d-1]}{x[1:d-1] - \mbf{1} x[d]} +  x[d]\\
&= \ip{(q[1:d-1], 1)}{(x[1:d-1] - \mbf{1} x[d] , x[d])}
\end{align*}
Since the above we have for any distribution over the $d$ arms $q$, and any loss vector $x$, we see that solving the linear bandit problem where learner picks from simplex and adversary picks from $[0,s]^d$ is equivalent to the linear bandit game where learner picks vectors from set $\F'$ and adversary picks vectors from set $\X'$ where
$$
\F' = \left\{(f,1) : f \in \reals^{d-1} \textrm{ s.t. }\forall i \in [d-1], f[i] \ge 0 , \sum_{i=1}^{d-1} f[i] \le 1\right\}
$$
and $\X' = \left\{(x[1:d-1] - \mbf{1} x[d] , x[d]) : x \in \X\right\}$.  Now we claim that the function  $\cR(f) = - \sum_{i=1}^{d-1} \log(f[i]) - \log(1 - \sum_{i=1}^{d-1} f[i])$ is a self-concordant barrier of the set $\F'$. To see this first note that the function $\tilde{\cR}(f[1:d-1]) = - \sum_{i=1}^{d-1} \log(f[i]) - \log(1 - \sum_{i=1}^{d-1} f[i])$ is a self-concordant barrier on the set $\{ f \in \reals^{d-1} : \forall i \in [d-1] f[i] \ge 0  , \sum_{i=1}^{d-1} f[i] \le 1\}$. Now since the function $\cR$ is simply the same as the function  $\tilde{\cR}$ applied only on the first $d-1$ coordinates of the input it is easy to see that $\cR$ is a self concordant barrier on $\F'$. Hence using Lemma \ref{lem:banditlstar} we can conclude that for the SCRiBLe algorithm with this reduction with any choice of $\eta > 0$ and any $q^* \in \Delta([d])$, 
\begin{align}
\E{\sum_{t=1}^T  \inner{q_t, x_t}} &\le \frac{1}{1 - (2 s d^2) \eta } \left(\sum_{t=1}^T \inner{q^*,x_t} +  d {\eta}^{-1} \max_{i \in [d]} \log(1/q^*[i]) \right) \notag \\
& \le \frac{1}{1 - (2 s d^2) \eta } \left(\inf_{q \in \Delta([d])}\sum_{t=1}^T \inner{q,x_t} + 1 +  d {\eta}^{-1}\log(d T) \right) \notag  \\
& = \frac{1}{1 - (2 s d^2) \eta } \left(\inf_{j \in [d]}\sum_{t=1}^T \inner{e_{j},x_t} + 1 +  d {\eta}^{-1}\log(d T) \right) \label{eq:simpbandbnd}
\end{align}
where the last step obtained by picking $q^* = (1 - 1/T) e_{j^*} + \sum_{i \ne j^*} (1/(d-1) T) e_i$ with $j^* = \argmin{j \in [d]} \sum_{t=1}^T \ip{e_j}{x_t}$.

%%%%%%%%%%%%%%%%%%%%%%%%%%%%%%%
 Thus we have a linear bandit algorithm over the simplex with the bound given in Equation \eqref{eq:simpbandbnd}. Now we claim that this algorithm can be used for solving multi-armed bandit problem.

\paragraph{Using linear bandit algorithm over simplex for multi-armed bandit problem:}
 We claim that the algorithm we have developed for the simplex case can be used for the multi-armed bandit problem. To see this note first that for any choice of $q_1,\ldots,q_T \in \Delta([d])$ and any choice of $x_1,\ldots,x_T$, 
\begin{align*}
\E{\sum_{t=1}^T \ip{q_t}{x_t} } - \inf_{q \in \Delta([d])} \ip{q}{x_t} & = \E{\sum_{t=1}^T \Es{j_t \sim q_t}{\ip{e_{j_t}}{x_t}} } - \inf_{i \in [d]} \ip{e_i}{x_t}\\
& = \E{\sum_{t=1}^T \ip{e_{j_t}}{x_t}  - \inf_{i \in [d]} \ip{e_i}{x_t}}
\end{align*}
Hence this shows that if we have an algorithm that outputs $q_1,\ldots,q_T$ then on each round by sampling the arm to pick from $q_t$ we get the same regret bound. However note that to run a bandit algorithm over the simplex we needed to be able to observe $\ip{q_t}{x_t}$, while in reality we only observe $\ip{e_{j_t}}{x_t}$. There is an easy remedy for this. Note that we needed to observe $\ip{q_t}{x_t}$ only to produce the unbiased estimate $\tilde{x}_t := d \left(\inner{q_t, x_t} \right)\varepsilon_t \lambda_{i_t}^{1/2} \cdot \ev{i_t}$. However,  $d \left(\inner{q_t, x_t} \right)\varepsilon_t \lambda_{i_t}^{1/2} \cdot \ev{i_t} = \Es{j_t \sim q_t}{d \left(\inner{e_{j_t}, x_t} \right)\varepsilon_t \lambda_{i_t}^{1/2} \cdot \ev{i_t}}$. Hence,  $d \left(\inner{e_{j_t}, x_t} \right)\varepsilon_t \lambda_{i_t}^{1/2} \cdot \ev{i_t}$ is also an unbiased estimate of $\tilde{x}_t$ and so the algorithm can simply use $\ip{e_{j_t}}{x_t}$ to build the estimates while enjoying the same bound in expectation. Thus, SCRiBLe for multi-armed bandit enjoys the bound
$$
\En \left\{ \sum_{t=1}^T \inner{e_{j_t}, x_t} \right\} \leq \frac{1}{1-4 \eta s d^2}\left(\inf_{j\in[d]}\sum_{t=1}^T \inner{e_{j}, x_t} + d \eta^{-1}\log (d T)\right)
$$
which concludes the proof.
\end{proof}